\documentclass[number]{ReportTemplate}
\usepackage{utopia}
\usepackage{amssymb}
\usepackage{amsfonts}
\usepackage[fleqn]{amsmath}
\usepackage{mathtools}
\usepackage{amsthm}
\usepackage{bm}
\usepackage{graphicx}
\usepackage{algorithmic}
\usepackage{hyperref}
\mathtoolsset{showonlyrefs,showmanualtags}

\newcommand{\expect}[1]{\mathbb{E}[\kern-0.15em[ #1 ]\kern-0.14em]}

\newtheorem{algorithm}{Algorithm}

\newenvironment{proof2}
{\par\noindent\textbf{Proof of Theorem \ref{analysis_approach}.}\ \enspace\ignorespaces\begin{allowdisplaybreaks}}
{\end{allowdisplaybreaks}\hspace{\stretch{1}}$\square$\\}

\begin{document}

\begin{frontmatter}

\title{Analyzing Evolutionary Optimization in Noisy Environments}
\author{Chao Qian}
\ead{qianc@lamda.nju.edu.cn}
\author{Yang Yu\corref{cor1}}
\ead{yuy@nju.edu.cn}
\author{Zhi-Hua Zhou}
\ead{zhouzh@nju.edu.cn}
\cortext[cor1]{Corresponding author}
\address{National Key Laboratory for Novel Software Technology\\
Nanjing University, Nanjing 210023, China}

\begin{abstract}
\begin{quote}
Many optimization tasks have to be handled in noisy environments, where we cannot obtain the exact evaluation of a solution but only a noisy one. For noisy optimization tasks, evolutionary algorithms (EAs), a kind of stochastic metaheuristic search algorithm, have been widely and successfully applied. Previous work mainly focuses on empirical studying and designing EAs for noisy optimization, while, the theoretical counterpart has been little investigated. In this paper, we investigate a largely ignored question, i.e., whether an optimization problem will always become harder for EAs in a noisy environment. We prove that the answer is negative, with respect to the measurement of the expected running time. The result implies that, for optimization tasks that have already been quite hard to solve, the noise may not have a negative effect, and the easier a task the more negatively affected by the noise. On a representative problem where the noise has a strong negative effect, we examine two commonly employed mechanisms in EAs dealing with noise, the \emph{re-evaluation} and the \emph{threshold selection} strategies. The analysis discloses that the two strategies, however, both are not effective, i.e., they do not make the EA more noise tolerant. We then find that a small modification of the threshold selection allows it to be proven as an effective strategy for dealing with the noise in the problem.
\end{quote}
\end{abstract}


\end{frontmatter}

\newpage

\section{Introduction}

Optimization tasks often encounter noisy environments. For example, in airplane design, every prototype is evaluated by simulations so that the evaluation result may not be perfect due to the simulation error; and in machine learning, a prediction model is evaluated only on a limited amount of data so that the estimated performance is shifted from the true performance. Noisy environments could change the property of an optimization problem, thus traditional optimization techniques may have low efficacy. While, evolutionary algorithms (EAs) \cite{back:96} have been widely and successfully adopted for noisy optimization tasks \cite{freitas2003survey,ma2006evolutionary,chang2006new,chang2006automated}.

EAs are a kind of randomized metaheuristic optimization algorithms, inspired by natural phenomena including evolution of species, swarm cooperation, immune system, etc. EAs typically involve a cycle of three stages: reproduction stage produces new solutions based on the currently maintained solutions; evaluation stage evaluates the newly generated solutions; selection stage wipes out bad solutions. An inspiration of using EAs for noisy optimization is that the corresponding natural phenomena have been processed successfully in noisy environments, and hence the algorithmic simulations are also likely to be able to handle noise. Besides, improved mechanisms have been invented for better handling noise. Two representative strategies are \emph{re-evaluation} and \emph{threshold selection}: by the re-evaluation strategy \cite{jin2005evolutionary,goh2007investigation}, whenever the fitness (also called cost or objective value) of a solution is required, EAs make an independent evaluation of the solution despite of whether the solution has been evaluated before, such that the fitness is smoothed; by the threshold selection strategy \cite{markon2001thresholding,beielstein2002threshold,bartz2005new}, in the selection stage EAs accept a newly generated solution only if its fitness is larger than the fitness of the old solution by at least a threshold, such that the risk of accepting a bad solution due to noise is reduced.


An assumption implied by using a noise handling mechanism in EAs is that the noise makes the optimization harder, so that a better handling mechanism can reduce the negative effect by the noise \cite{fitzpatrick1988genetic,beyer2000evolutionary,rudolph2001partial,arnold2003comparison}. This paper firstly investigates if this assumption is true. We start by presenting an experimental evidence using (1+1)-EA optimizing the hardest case in the pseudo-Boolean function class \cite{qian2012algorithm}. Experiment results indicate that the noise, however, makes the optimization easier rather than harder, under the measurement of expected running time.

Following the experiment evidence, we then derive sufficient theoretical conditions, under which the noise will make the optimization easier or harder. By filling the conditions, we present proofs that, for the (1+$\lambda$)-EA (a class of EAs employing offspring population size $\lambda$), the noise will make the optimization easier on the hardest case in the pseudo-Boolean function class, while harder on the easiest case. The proofs imply that we need to take care of the noise only when the optimization is moderately or less complex, and ignore this issue when the optimization task itself is quite hard.

For the situations where the noise needs to be cared, this paper examines the re-evaluation and the threshold selection strategies for their \emph{polynomial noise tolerance} (PNT). For a kind of noise, the PNT of an EA is the maximum noise level such that the expected running time of the algorithm is polynomial. The closer the PNT is to 1, the better the noise tolerance is. Taking the easiest pseudo-Boolean function case as the representative problem, we analyze the PNT for different configurations of the (1+1)-EA with respect to the one-bit noise, whose level is characterized by the noise probability. For the (1+1)-EA (without any noise handling strategy), we prove that the PNT has a lower bound $1-\frac{1}{\Omega(poly(n))}$ and an upper bound $1-\frac{1}{O(2^npoly(n))}$. Since the (1+1)-EA with re-evaluation has the PNT $\Theta(\frac{\log n}{n})$ \cite{droste2004analysis}, it is surprisingly that the re-evaluation makes the PNT much worse. We further prove that for the (1+1)-EA with re-evaluation using threshold selection, when the threshold is 1, the PNT is not less than $\frac{1}{2e}$, and when the threshold is 2, the PNT has a lower bound $1-\frac{1}{\Omega(poly(n))}$ and an upper bound $1-\frac{1}{O(2^npoly(n))}$. The PNT bounds indicate that threshold selection improves the re-evaluation strategy, however, no improvements from the (1+1)-EA are found. We then introduce a small modification into the threshold selection strategy to turn the original hard threshold to be a smooth threshold. We prove that with the smooth threshold selection strategy the PNT is $1$, i.e., the (1+1)-EA is always a polynomial algorithm disregard the probability of one-bit noise on the problem.

The rest of this paper is organized as follows. Section 2 introduces some background. Section 3 shows that the noise may not always be bad, and presents a sufficient condition for that. Section 4 analyzes noise handling strategies. Section 5 concludes.

\section{Background}

\subsection{Noisy Optimization}

A general optimization problem can be represented as $
\arg\max\nolimits_{x} f(x)$, where the objective $f$ is also called fitness in the context of evolutionary computation. In real-world optimization tasks, the fitness evaluation for a solution is usually disturbed by noise, and consequently we can not obtain the exact fitness value but only a noisy one. In this paper, we will involve the following kinds of noise, and we will always denote $f^N(x)$ and $f(x)$ as the noisy and true fitness of a solution $x$, respectively.
\begin{description}
    \item[additive noise] $f^N(x)=f(x)+ \delta$, where $\delta$ is uniformly selected from $[\delta_1,\delta_2]$ at random.
    \item[multiplicative noise] $f^N(x)=f(x)\cdot \delta$, where $\delta$ is uniformly selected from $[\delta_1,\delta_2]$ at random.
    \item[one-bit noise] $f^N(x)=f(x)$ with probability $(1-p_n)$ $(0\leq p_n \leq 1)$; otherwise, $f^N(x)=f(x')$, where $x'$ is generated by flipping a uniformly randomly chosen bit of $x \in \{0,1\}^n$. This noise is for problems where solutions are represented in binary strings.
\end{description}
Additive and multiplicative noise has been often used for analyzing the effect of noise \cite{beyer2000evolutionary,jin2005evolutionary}. One-bit noise is specifically for optimizing pseudo-Boolean problems over $\{0,1\}^n$, and also the investigated noise in the only previous work for analyzing running time of EAs in noisy optimization \cite{droste2004analysis}. For one-bit noise, $p_n$ controls the noise level. In this paper we assume that the parameters of the environment (i.e., $p_n$, $\delta_1$ and $\delta_2$) do not change over time. 

It is possible that a large noise could make an optimization problem extremely hard for particular algorithms. We are interested in the noise level, under which an algorithm could be ``tolerant'' to have polynomial running time. We define the polynomial noise tolerance (PNT) as Definition \ref{PNT}, which characterizes the maximum noise level for allowing a polynomial expected running time. Note that, the noise level can be measured by the adjusting parameter, e.g., $\delta_1, \delta_2$ for the additive and multiplicative noise, and $p_n$ for the one-bit noise. We will study the PNT of EAs for analyzing the effectiveness of noise handling strategies.

\begin{Def}[Polynomial Noise Tolerance (PNT)]\label{PNT}
The polynomial noise tolerance of an algorithm on a problem, with respect to a kind of noise, is the maximum noise level such that the algorithm has expected running time polynomial to the problem size.
\end{Def}

\subsection{Evolutionary Algorithms}

Evolutionary algorithms (EAs) \cite{back:96} are a kind of population-based metaheuristic optimization algorithms. Although there exist many variants, the common procedure of EAs can be described as follows:\vspace{0.5em}\\
1. \; Generate an initial set of solutions (called population);\\
2. \; Reproduce new solutions from the current population;\\
3. \; Evaluate the newly generated solutions;\\
4. \; Update the population by removing bad solutions;\\
5. \; Repeat steps 2-5 until some criterion is met.\vspace{0.5em}

The (1+1)-EA, as in Algorithm \ref{(1+1)-EA}, is a simple EA for maximizing pseudo-Boolean problems over $\{0,1\}^n$, which reflects the common structure of EAs. It maintains only one solution, and repeatedly improves the current solution by using bit-wise mutation (i.e., the 3rd step of Algorithm \ref{(1+1)-EA}). It has been widely used for the running time analysis of EAs, e.g., \cite{YaoAI01,droste2002analysis}.

\begin{algorithm}[(1+1)-EA]\label{(1+1)-EA} Given pseudo-Boolean function $f$ with solution length $n$, it consists of the following steps:\\
    \begin{tabular}{ll}
    1. & $x:=$ randomly selected from $\{0,1\}^{n}$.\\
    2. & Repeat until the termination condition is met\\
    3. & \quad $x':=$ flip each bit of $x$ with probability $p$. \\
    4. &\quad if {$f(x') \geq f(x)$} \\
    5. &\quad \quad $x:=x'$.
    \end{tabular}\\
where $p \in (0,0.5)$ is the mutation probability.
\end{algorithm}

The (1+$\lambda$)-EA, as in Algorithm \ref{(1+lambda)-EA}, applies an offspring population size $\lambda$. In each iteration, it first generates $\lambda$ offspring solutions by independently mutating the current solution $\lambda$ times, and then selects the best solution from the current solution and the offspring solutions as the next solution. It has been used to disclose the effect of offspring population size by running time analysis \cite{jansen2005choice,neumann2007randomized}. Note that, (1+1)-EA is a special case of (1+$\lambda$)-EA with $\lambda=1$.

\begin{algorithm}[(1+$\lambda$)-EA]\label{(1+lambda)-EA} Given pseudo-Boolean function $f$ with solution length $n$, it consists of the following steps:\\
    \begin{tabular}{ll}
    1. & $x:=$ randomly selected from $\{0,1\}^{n}$.\\
    2. & Repeat until the termination condition is met\\
    3. & \quad $i:=1$.\\
    4. & \quad Repeat until $i>\lambda$. \\
    5. & \quad \quad $x_i:=$ flip each bit of $x$ with probability $p$. \\
    6. & \quad \quad $i:=i+1$. \\
    7. &\quad $x=\arg\max_{x'\in\{x,x_1,\ldots,x_{\lambda}\}} f(x').$
    \end{tabular}\\
where $p \in (0,0.5)$ is the mutation probability.
\end{algorithm}

The running time of EAs is usually defined as the number of fitness evaluations (i.e., computing $f(\cdot)$) until an optimal solution is found for the first time, since the fitness evaluation is the computational process with the highest cost of the algorithm \cite{YaoAI01,Yu:Zhou:08}.

\subsection{Markov Chain Modeling}

We will analyze EAs by modeling them as Markov chains in this paper. Here, we first give some preliminaries.

EAs generate solutions only based on their currently maintained solutions, thus, they can be modeled and analyzed as Markov chains, e.g., \cite{YaoAI01,Yu:Zhou:08}. A Markov chain $\{\xi_t\}^{+\infty}_{t=0}$ modeling an EA is constructed by taking the EA's population space $\mathcal{X}$ as the chain's state space, i.e. $\xi_t \in \mathcal{X}$. Let $\mathcal{X}^* \subset \mathcal{X}$ denote the set of all optimal populations, which contains at least one optimal solution. The goal of the EA is to reach $\mathcal{X}^*$ from an initial population. Thus, the process of an EA seeking $\mathcal{X}^*$ can be analyzed by studying the corresponding Markov chain.

A Markov chain $\{\xi_t\}_{t=0}^{+\infty}$ $(\xi_t \in \mathcal{X})$ is a random process, where $\forall t \geq 0$, $\xi_{t+1}$ depends only on $\xi_t$. A Markov chain $\{\xi_t\}^{+\infty}_{t=0}$ is said to be homogeneous, if $\forall t \geq 0,\forall x,y \in \mathcal{X}$:
\begin{equation}
\begin{aligned}\label{homogeneous}
&P(\xi_{t+1}=y|\xi_t=x)=P(\xi_1=y|\xi_0=x).
\end{aligned}
\end{equation}
In this paper, we always denote $\mathcal{X}$ and $\mathcal{X}^*$ as the state space and the optimal state space of a Markov chain, respectively.

Given a Markov chain $\{\xi_t\}^{+\infty}_{t=0}$ and $\xi_{\hat{t}}=x$, we define the first hitting time (FHT) of the chain as a random variable $\tau$ such that $\tau=\min\{t|\xi_{\hat{t}+t} \in \mathcal{X}^*,t\geq0\}$. That is, $\tau$ is the number of steps needed to reach the optimal state space for the first time starting from $\xi_{\hat{t}}=x$. The mathematical expectation of
$\tau$, $\expect{\tau | \xi_{\hat{t}}=x}=\sum\nolimits^{\infty}_{i=0} iP(\tau=i)$, is called the expected first hitting time (EFHT) of this chain starting from $\xi_{\hat{t}}=x$. If $\xi_{0}$ is drawn from a distribution $\pi_{0}$, $\expect{\tau | \xi_{0}\sim \pi_0} = \sum\nolimits_{x\in \mathcal{X}} \pi_{0}(x)\expect{\tau | \xi_{0}=x}$ is called the expected first hitting time of the Markov chain over the initial distribution $\pi_0$.

For the corresponding EA, the running time is the numbers of calls to the fitness function until meeting an optimal solution for the first time. Thus, the \emph{expected running time} starting from $\xi_0$ and that starting from $\xi_0 \sim \pi_0$ are respectively equal to
\begin{equation}
\begin{aligned}\label{runtime}
N_1+N_2\cdot \expect{\tau | \xi_{0}} && \text{and} && N_1+N_2\cdot \expect{\tau | \xi_{0} \sim \pi_0},
\end{aligned}
\end{equation}
where $N_1$ and $N_2$ are the number of fitness evaluations for the initial population and each iteration, respectively. For example, for (1+1)-EA, $N_1=1$ and $N_2=1$; for (1+$\lambda$)-EA, $N_1=1$ and $N_2=\lambda$. Note that, when involving the expected running time of an EA on a problem in this paper, if the initial population is not specified, it is the expected running time starting from a uniform initial distribution $\pi_u$, i.e., $N_1+N_2 \cdot \expect{\tau | \xi_{0} \sim \pi_u}=N_1+N_2 \cdot\sum\nolimits_{x\in \mathcal{X}} \frac{1}{|\mathcal{X}|}\expect{\tau | \xi_{0}=x}$.

The following two lemmas on the EFHT of Markov chains \cite{Freidlin:97} will be used in this paper.

\begin{lemma}\label{lem_onestep}
        Given a Markov chain $\{\xi_t\}^{+\infty}_{t=0}$, we have
        \begin{equation}
        \begin{aligned}
        &\forall x \in \mathcal{X}^*: \expect{\tau | \xi_t=x}=0; \\
        &\forall x\notin \mathcal{X}^*: \expect{\tau | \xi_t=x}=1+\sum\nolimits_{y\in \mathcal{X}} P(\xi_{t+1}=y | \xi_t=x)\expect{\tau | \xi_{t+1}=y}.
        \end{aligned}
        \end{equation}
\end{lemma}
\begin{lemma}\label{lem_homo}
        Given a homogeneous Markov chain $\{\xi_t\}^{+\infty}_{t=0}$, it holds
        $$\forall t_1, t_2 \geq 0, x \in \mathcal{X}: \expect{\tau |\xi_{t_1}=x} = \expect{\tau|  \xi_{t_2}=x}.$$
\end{lemma}

For analyzing the EFHT of Markov chains, drift analysis \cite{YaoAI01,he2004study} is a commonly used tool, which will also be used in this paper. To use drift analysis, it needs to construct a function $V(x)\;(x \in \mathcal{X})$ to measure the distance of a state $x$ to the optimal state space $\mathcal{X}^*$. The distance function $V(x)$ satisfies that $V(x \in \mathcal{X}^*)=0$ and $V(x \notin \mathcal{X}^*)>0$. Then, by investigating the progress on the distance to $\mathcal{X}^*$ in each step, i.e., $\expect{V(\xi_t)-V(\xi_{t+1}) | \xi_t}$, an upper (lower) bound of the EFHT can be derived through dividing the initial distance by a lower (upper) bound of the progress.

\begin{lemma}[Drift Analysis \cite{YaoAI01,he2004study}]\label{drift}
Given a Markov chain $\{\xi_t\}^{+\infty}_{t=0}$ and a distance function $V(x)$, if it satisfies that for any $t \geq 0$ and any $\xi_t$ with $V(\xi_t) > 0$,
$$
0<c_l \leq \expect{V(\xi_t)-V(\xi_{t+1}) | \xi_t} \leq c_u,
$$
then the EFHT of this chain satisfies that
$$
V(\xi_0)/c_u \leq \expect{\tau | \xi_0} \leq V(\xi_0)/c_l,
$$
where $c_l,c_u$ are constants.
\end{lemma}

\subsection{Pseudo-Boolean Functions}

The pseudo-Boolean function class in Definition \ref{def_Boolean} is a large function class which only requires the solution space to be $\{0,1\}^n$ and the objective space to be $\mathbb{R}$. Many well-known NP-hard problems (e.g., the vertex cover problem and the 0-1 knapsack problem) belong to this class. Diverse pseudo-Boolean problems with different structures and difficulties have been used for analyzing the running time of EAs, and then to disclose properties of EAs, e.g., \cite{droste:jansen:wegener:98,YaoAI01,droste2002analysis}. Note that, we consider only maximization problems in this paper since minimizing $f$ is equivalent to maximizing $-f$.

\begin{definition}[Pseudo-Boolean Function]\label{def_Boolean}
    A function in the pseudo-Boolean function class has the form:
    $
        f:\{0,1\}^n \rightarrow \mathbb{R}.
    $
\end{definition}

I$_{hardest}$ (or called Trap) problem in Definition \ref{def_trap} is a special instance in this class, which is to maximize the number of 0 bits of a solution except the global optimum $11\ldots1$ (briefly denoted as $1^n$). Its optimal function value is $2n$, and the function value for any non-optimal solution is not larger than 0. It has been widely used in the theoretical analysis of EAs, and the expected running time of (1+1)-EA with mutation probability $\frac{1}{n}$ has been proved to be $\Theta(n^n)$ \cite{droste2002analysis}. It has also been recognized as the hardest instance in the pseudo-Boolean function class with a unique global optimum for the (1+1)-EA \cite{qian2012algorithm}.

\begin{definition}[I$_{hardest}$ Problem]\label{def_trap}
    I$_{hardest}$ Problem of size $n$ is to find an $n$ bits binary string $x^*$ such that
    $$
        x^*=\mathop{\arg\max}\nolimits_{x \in \{0,1\}^n} \big( f(x)=3n\prod\nolimits^n_{i=1}x_i -\sum\nolimits^{n}_{i=1} x_i\big),
    $$
    where $x_i$ is the $i$-th bit of a solution $x \in \{0,1\}^n$.
\end{definition}

I$_{easiest}$ (or called OneMax) problem in Definition \ref{def_onemax} is to maximize the number of 1 bits of a solution. The optimal solution is $1^n$, which has the maximal function value $n$. The running time of EAs has been well studied on this problem \cite{YaoAI01,droste2002analysis,sudholt2011new}. Particularly, the expected running time of (1+1)-EA with mutation probability $\frac{1}{n}$ on it has been proved to be $\Theta(n \log n)$ \cite{droste2002analysis}. It has also been recognized as the easiest instance in the pseudo-Boolean function class with a unique global optimum for the (1+1)-EA \cite{qian2012algorithm}.

\begin{definition}[I$_{easiest}$ Problem]\label{def_onemax}
    I$_{easiest}$ Problem of size $n$ is to find an $n$ bits binary string $x^*$ such that
    $$
        x^*=\mathop{\arg\max}\nolimits_{x \in \{0,1\}^n} \big( f(x)=\sum\nolimits^{n}_{i=1} x_i\big),
    $$
    where $x_i$ is the $i$-th bit of a solution $x \in \{0,1\}^n$.
\end{definition}

\section{Noise is Not Always Bad}

\subsection{Empirical Evidence}

It has been observed that noisy fitness evaluation can make an optimization harder for EAs, since it may make a bad solution have a ``better" fitness, and then mislead the search direction of EAs. Droste \cite{droste2004analysis} proved that the running time of (1+1)-EA can increase from polynomial to exponential due to the presence of noise. However, when studying the running time of (1+1)-EA solving the hardest case I$_{hardest}$ in the pseudo-Boolean function class, we have observed oppositely that noise can also make an optimization easier for EAs, which means that the presence of the noise decreases the running time of EAs for finding the optimal solution.

For I$_{hardest}$ problem over $\{0,1\}^n$, there are $2^n$ possible solutions, which are denoted by their corresponding integer values $0,1,\ldots,2^n-1$, respectively. Then, we estimate the expected running time of (1+1)-EA maximizing I$_{hardest}$ when starting from every solution. For each initial solution, we repeat independent runs for 1000 times, and then the average running time is recorded as an estimation of the expected running time (briefly called as ERT). We run (1+1)-EA without noise, with additive noise and with multiplicative noise, respectively. For the mutation probability of (1+1)-EA, we use the common setting $p=\frac{1}{n}$. For additive noise, $\delta_1=-n$ and $\delta_2=n$, and for multiplicative noise, $\delta_1=0.1$ and $\delta_2=10$. The results for $n=3,4,5$ are plotted in Figure \ref{fig_ERT_helpful1}. We can observe that the curves by these two kinds of noise are always under the curve without noise, which shows that I$_{hardest}$ problem becomes easier for (1+1)-EA in a noisy environment. Note that, the three curves meet at the last point, since the initial solution $2^n-1$ is the optimal solution and then ERT $=1$.

\begin{figure*}[t!]\centering
\begin{minipage}[c]{0.33\linewidth}\centering
        \includegraphics[width=0.8\linewidth,height=0.65\linewidth]{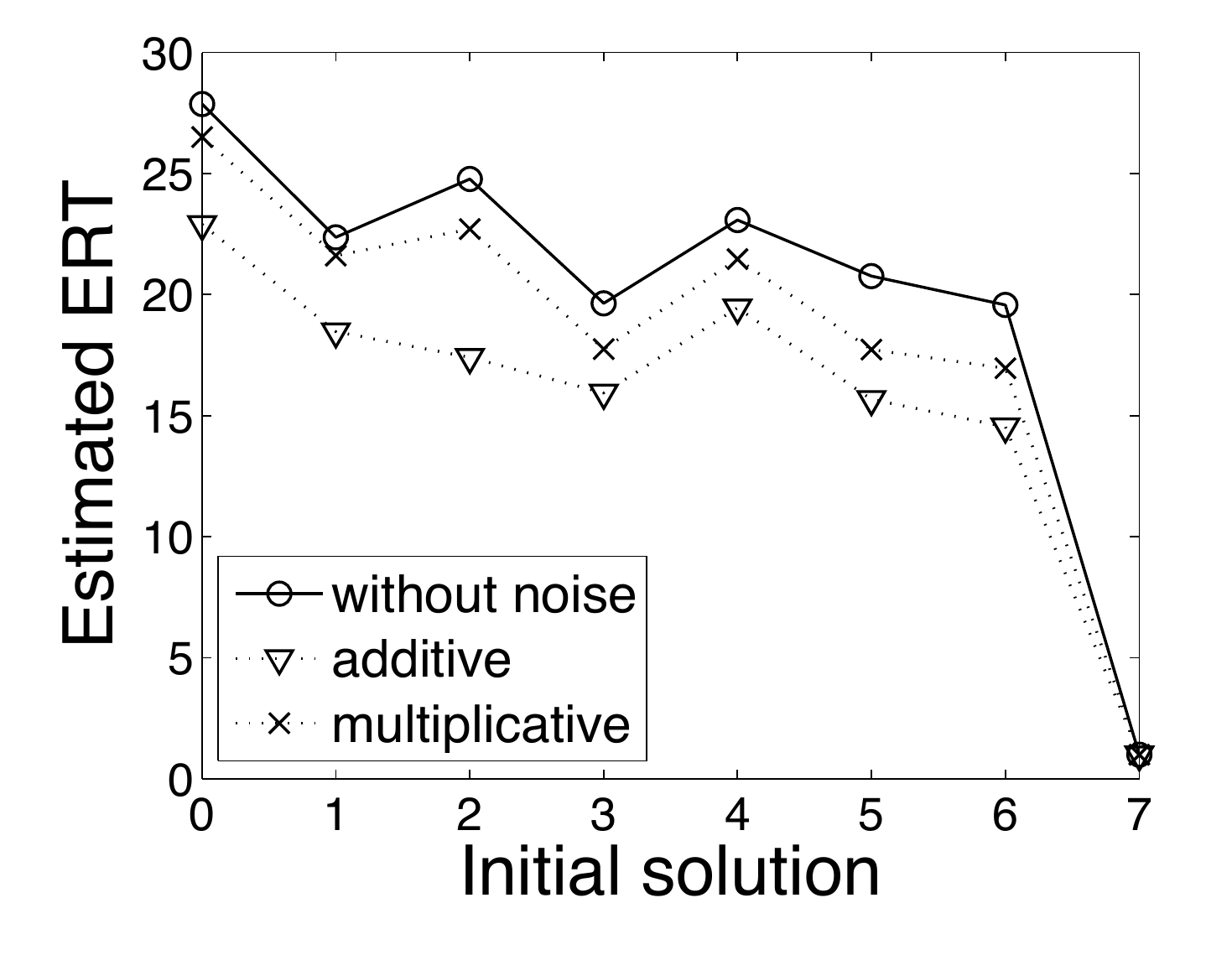}
\end{minipage}
\begin{minipage}[c]{0.33\linewidth}\centering
        \includegraphics[width=0.8\linewidth,height=0.65\linewidth]{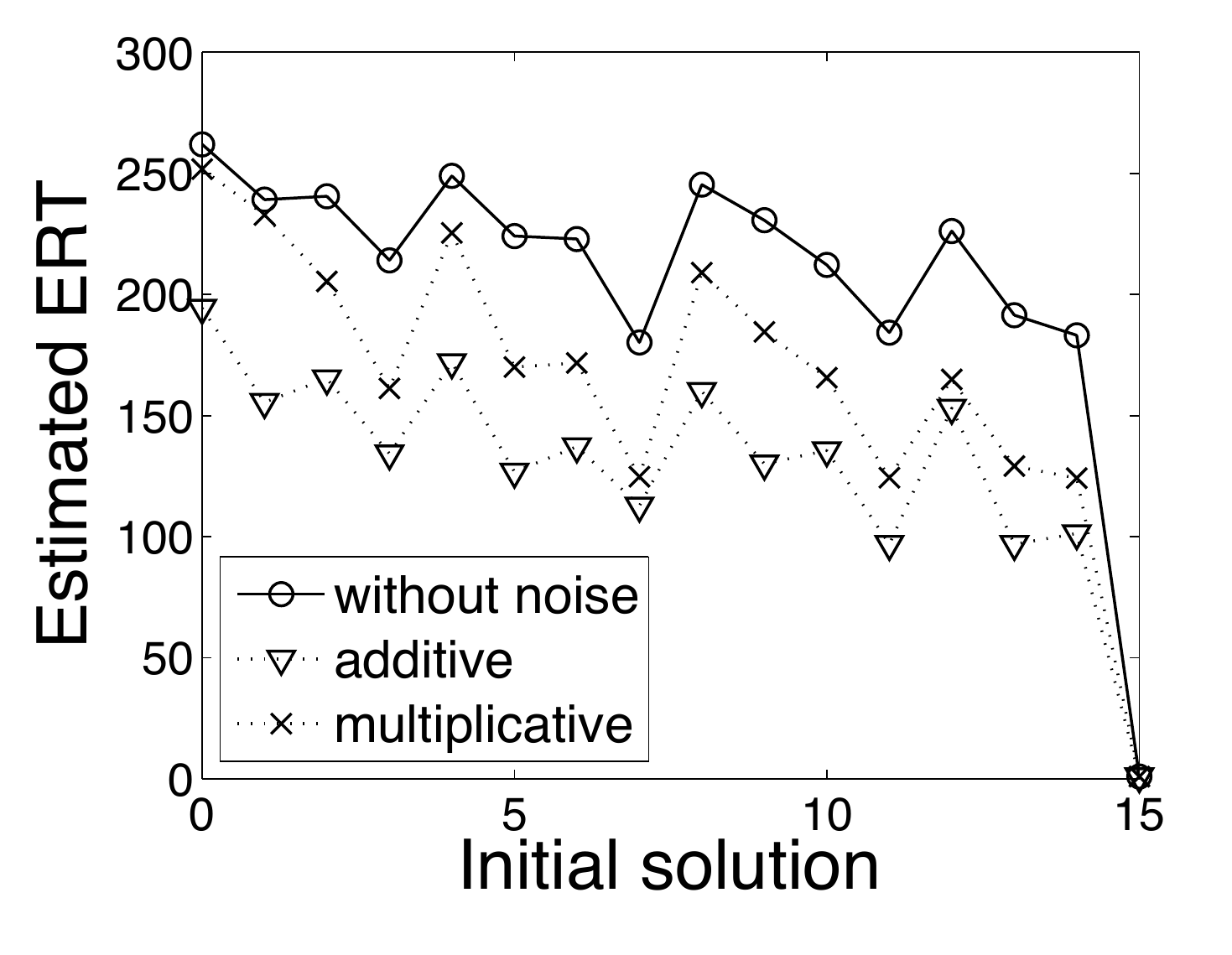}
\end{minipage}
\begin{minipage}[c]{0.33\linewidth}\centering
        \includegraphics[width=0.8\linewidth,height=0.65\linewidth]{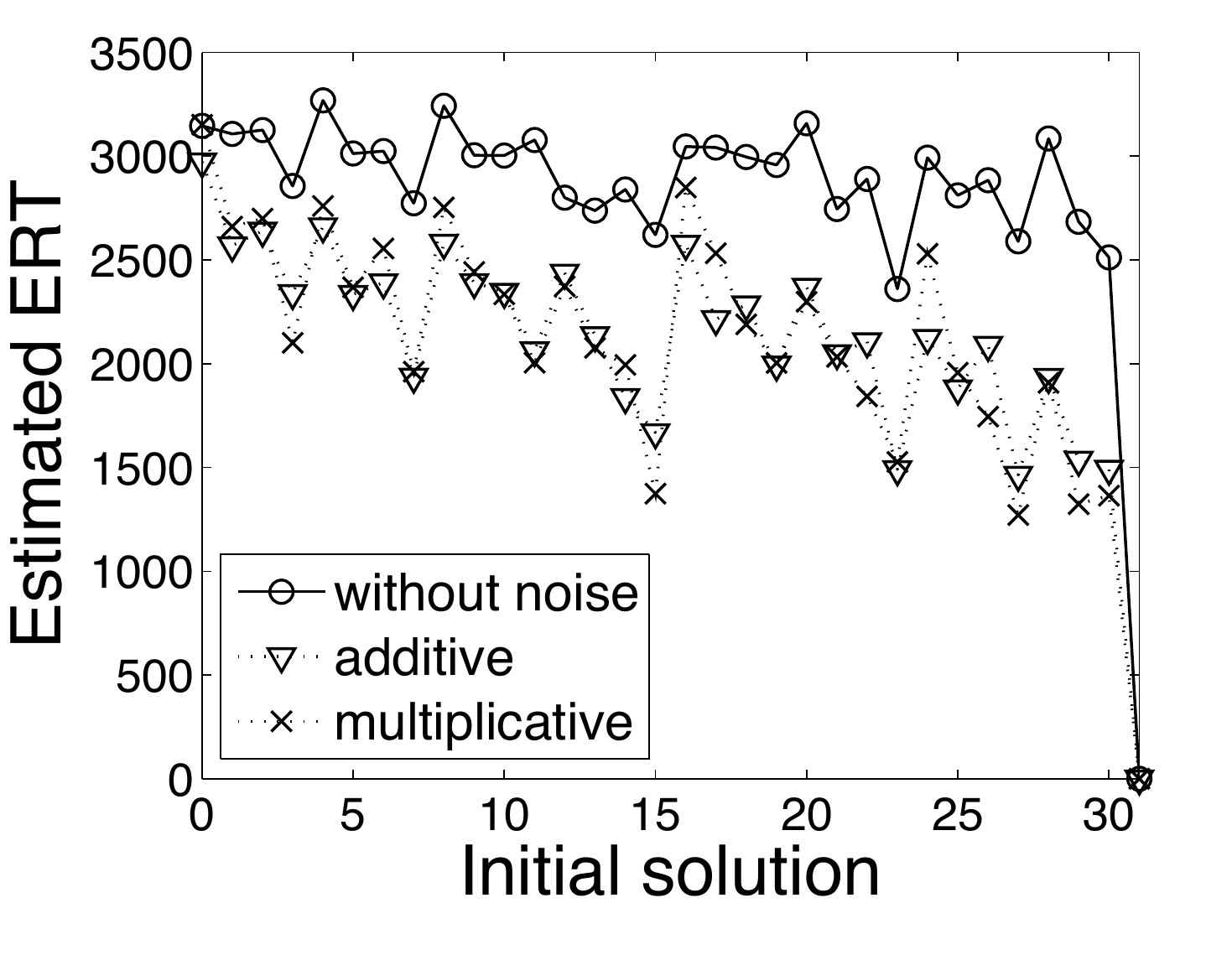}
\end{minipage}\\
\begin{minipage}[c]{0.33\linewidth}\centering
    \small(a) $n=3$
\end{minipage}
\begin{minipage}[c]{0.33\linewidth}\centering
    \small(b) $n=4$
\end{minipage}
\begin{minipage}[c]{0.33\linewidth}\centering
    \small(c) $n=5$
\end{minipage}\\\vspace{-0.6em}
\caption{Estimated ERT comparison for (1+1)-EA solving I$_{hardest}$ problem with or without noise.}\label{fig_ERT_helpful1}
\end{figure*}

\subsection{A Sufficient Condition}

In this section, by comparing the expected running time of EAs with and without noise, we derive a sufficient condition under which the noise will make an optimization easier for EAs.

Most practical EAs employ time-invariant operators, thus we can model an EA without noise by a homogeneous Markov chain. While for an EA with noise, since noise may change over time, we can just model it by a Markov chain. Note that, the two EAs with and without noise are different only on whether the fitness evaluation is disturbed by noise, thus, they must have the same values on $N_1$ and $N_2$ for their running time Eq.\refeq{runtime}. Then, comparing their expected running time is equivalent to comparing the EFHT of their corresponding Markov chains.

We first define a partition of the state space of a homogeneous Markov chain based on the EFHT, and then define a jumping probability of a Markov chain from one state to one state space in one step. It is easy to see that $\mathcal{X}_0$ in Definition \ref{def_partition} is just $\mathcal{X}^*$, since $\expect{\tau|\xi_0 \in \mathcal{X}^*}=0$.

\begin{definition}[EFHT-Partition]\label{def_partition}
For a homogeneous Markov chain $\{\xi_t\}^{+\infty}_{t=0}$, the EFHT-Partition is a partition of $\mathcal{X}$ into non-empty subspaces $\{\mathcal{X}_0,\mathcal{X}_1,\ldots,\mathcal{X}_m\}$ such that
\begin{equation}
\begin{aligned}
&(1) \quad \forall x,y \in \mathcal{X}_i, \expect{\tau|\xi_0=x}=\expect{\tau|\xi_0=y};\\
&(2) \quad \expect{\tau|\xi_0 \in \mathcal{X}_0}<\expect{\tau|\xi_0 \in \mathcal{X}_1}< \ldots < \expect{\tau|\xi_0 \in \mathcal{X}_m}.
\end{aligned}
\end{equation}
\end{definition}

\begin{definition}\label{def_jump}
For a Markov chain $\{\xi_t\}^{+\infty}_{t=0}$, $P^t_{\xi}(x,\mathcal{X}')=\sum_{y \in \mathcal{X}'} P(\xi_{t+1}=y|\xi_{t}=x)$ is the probability of jumping from state $x$ to state space $\mathcal{X}'\subseteq \mathcal{X}$ in one step at time $t$.
\end{definition}

\begin{theorem}\label{analysis_approach}
Given an EA $\mathcal{A}$ and a problem $f$, let a Markov chain $\{\xi_t\}^{+\infty}_{t=0}$ and a homogeneous Markov chain $\{\xi'_t\}^{+\infty}_{t=0}$ model $\mathcal{A}$ running on $f$ with noise and without noise respectively, and denote $\{\mathcal{X}_0,\mathcal{X}_1,\ldots,\mathcal{X}_m\}$ as the EFHT-Partition of $\{\xi'_t\}^{+\infty}_{t=0}$, if for all $t\geq 0$, $x \in \mathcal{X}-\mathcal{X}_0$, and for all integers $i\in [0,m-1]$,
\begin{equation}
\begin{aligned}\label{analysis_condition}
&\sum\nolimits^i_{j=0}P^t_{\xi}(x,\mathcal{X}_j) \geq \sum\nolimits^{i}_{j=0} P^t_{\xi'}(x,\mathcal{X}_j),
\end{aligned}
\end{equation}
then noise makes $f$ easier for $\mathcal{A}$, i.e., for all $x \in \mathcal{X}$, $$\expect{\tau | \xi_{0}=x} \leq \expect{\tau' | \xi'_{0}=x}.$$
\end{theorem}

The condition of this theorem (i.e., Eq.\ref{analysis_condition}) intuitively means that the presence of noise leads to a larger probability of jumping into good states (i.e., $\mathcal{X}_j$ with small $j$ values), starting from which the EA needs less time for finding the optimal solution. For the proof, we need the following lemma, which is proved in the appendix.

\begin{lemma}\label{lemma_analysis_condition}
Let $m\;(m \geq 1)$ be an integer. If it satisfies that
\begin{equation}
\begin{aligned}
& (1)\quad \forall 0 \leq  i \leq m, P_i,Q_i\geq 0,\; \text{and} \; \sum\nolimits^{m}_{i=0}P_i=\sum\nolimits^{m}_{i=0}Q_i=1;\\
& (2)\quad 0\leq E_0<E_1<\ldots<E_m;\\
& (3)\quad \forall 0 \leq  k \leq m-1, \sum\nolimits^k_{i=0} P_i \leq \sum\nolimits^k_{i=0} Q_i,
\end{aligned}
\end{equation}
then it holds that
$$
\sum\nolimits^{m}_{i=0}P_i\cdot E_i \geq \sum\nolimits^{m}_{i=0}Q_i\cdot E_i.
$$
\end{lemma}

\begin{proof2}
We use Lemma \ref{drift} to derive a bound on $\expect{\tau|\xi_0}$, based on which this theorem holds.

For using Lemma \ref{drift} to analyze $\expect{\tau|\xi_0}$, we first construct a distance function $V(x)$ as \begin{equation}
\begin{aligned}\label{distance}
& \forall x \in \mathcal{X}, V(x)=\expect{\tau'|\xi'_0=x},
\end{aligned}
\end{equation}
which satisfies that $V(x \in \mathcal{X}^*)=0$ and $V(x \notin \mathcal{X}^*)>0$ by Lemma \ref{lem_onestep}.

Then, we investigate $\expect{V(\xi_t)-V(\xi_{t+1}) | \xi_t=x}$ for any $x$ with $V(x)>0$ (i.e., $x \notin \mathcal{X^*}$).
\begin{equation}
\begin{aligned}
&\expect{V(\xi_t)-V(\xi_{t+1}) | \xi_t=x}=V(x)-\expect{V(\xi_{t+1})|\xi_t=x}\\
&=V(x)-\sum\nolimits_{y \in \mathcal{X}} P(\xi_{t+1}=y|\xi_t=x) V(y)\\
&= \expect{\tau'|\xi'_0=x}-\sum\nolimits_{y \in \mathcal{X}} P(\xi_{t+1}=y|\xi_t=x) \expect{\tau'|\xi'_0=y} \quad (\text{by Eq.\refeq{distance}})\\
&=1+\sum\nolimits_{y \in \mathcal{X}} P(\xi'_{1}=y|\xi'_0=x) \expect{\tau'|\xi'_1=y}-\sum\nolimits_{y \in \mathcal{X}} P(\xi_{t+1}=y|\xi_t=x) \expect{\tau'|\xi'_0=y}\quad (\text{by Lemma \refeq{lem_onestep}})\\
&=1+\sum\nolimits_{y \in \mathcal{X}} P(\xi'_{t+1}=y|\xi'_t=x) \expect{\tau'|\xi'_0=y}-\sum\nolimits_{y \in \mathcal{X}} P(\xi_{t+1}=y|\xi_t=x) \expect{\tau'|\xi'_0=y}\\
&\quad (\text{by Eq.\refeq{homogeneous} and Lemma \ref{lem_homo}, since $\{\xi'_t\}^{+\infty}_{t=0}$ is homogeneous.})\\
&=1+\sum\nolimits^m_{j=0} (P^t_{\xi'}(x,\mathcal{X}_j)-P^t_{\xi}(x,\mathcal{X}_j)) \expect{\tau'|\xi'_0\in \mathcal{X}_j}.\quad (\text{by Definitions \ref{def_partition} and \ref{def_jump}})
\end{aligned}
\end{equation}
Since $\sum^m_{j=0} P^t_{\xi}(x,\mathcal{X}_j)=\sum^m_{j=0} P^t_{\xi'}(x,\mathcal{X}_j)=1$, $\expect{\tau'|\xi'_0 \in \mathcal{X}_j}$ increases with $j$ and Eq.\refeq{analysis_condition} holds, by Lemma \ref{lemma_analysis_condition}, we have
$$
\sum\nolimits^{m}_{j=0} P^t_{\xi'}(x,\mathcal{X}_j) \expect{\tau'|\xi'_0 \in \mathcal{X}_j} \geq \sum\nolimits^{m}_{j=0} P^t_{\xi}(x,\mathcal{X}_j) \expect{\tau'|\xi'_0 \in \mathcal{X}_j}.
$$
Thus, we have, for all $t \geq 0$, all $x \notin \mathcal{X}^*$,
$$
\expect{V(\xi_t)-V(\xi_{t+1}) | \xi_t=x}\geq 1.
$$

Thus, by Lemma \ref{drift}, we get for all $x \in \mathcal{X}$,
$$
\expect{\tau|\xi_0=x} \leq V(x)=\expect{\tau'|\xi'_0=x}, \quad \text{(the `$=$' is by Eq.\ref{distance})}
$$
which implies that noise leads to less time for finding the optimal solution, i.e., noise makes optimization easier.
\end{proof2}


We prove below that the experimental example satisfies this sufficient condition. We consider (1+$\lambda$)-EA, which covers (1+1)-EA and is much more general. Let $\{\xi_t\}^{+\infty}_{t=0}$ and $\{\xi'_t\}^{+\infty}_{t=0}$ model (1+$\lambda$)-EA with and without noise for maximizing I$_{hardest}$ problem, respectively. For I$_{hardest}$ problem, it is to maximize the number of 0 bits except the optimal solution $1^n$. It is not hard to see that the EFHT $\expect{\tau'|\xi'_0=x}$ only depends on $|x|_0$ (i.e., the number of 0 bits). We denote $\mathbb{E}_1(j)$ as $\expect{\tau'|\xi'_0=x}$ with $|x|_0=j$. The order of $\mathbb{E}_1(j)$ is showed in Lemma \ref{CFHT_Trap}, the proof of which is in the Appendix.

\begin{lemma}\label{CFHT_Trap}
For any mutation probability $0<p<0.5$, it holds that $\mathbb{E}_1(0)< \mathbb{E}_1(1)< \mathbb{E}_1(2)< \ldots < \mathbb{E}_1(n).$
\end{lemma}

\begin{theorem}\label{theo_helpful_case1}
Either additive noise with $\delta_2-\delta_1 \leq 2n$ or multiplicative noise with $\delta_2> \delta_1 >0$ makes I$_{hardest}$ problem easier for (1+$\lambda$)-EA with mutation probability less than 0.5.
\end{theorem}
\begin{proof}
The proof is by showing that the condition of Theorem \ref{analysis_approach} (i.e., Eq.\ref{analysis_condition}) holds here. By Lemma~\ref{CFHT_Trap}, the EFHT-Partition of $\{\xi'_t\}^{+\infty}_{t=0}$ is $\mathcal{X}_i=\{x \in \{0,1\}^n | |x|_0=i\} \;(0\leq i\leq n)$ and $m$ in Theorem \ref{analysis_approach} equals to $n$ here. Let $f^N(x)$ and $f(x)$ denote the noisy and true fitness, respectively.

For any $x \in \mathcal{X}_{k}\;(k \geq 1)$, we denote $P(0)$ and $P(j)\;(1 \leq j \leq n)$ as the probability that for the $\lambda$ offspring solutions $x_1,\ldots,x_{\lambda}$ generated by bit-wise mutation on $x$, $\min\{|x_1|_0,\ldots,|x_{\lambda}|_0\}=0$ (i.e., the least number of 0 bits is 0), and $\min\{|x_1|_0,\ldots,|x_{\lambda}|_0\}>0 \wedge \max\{|x_1|_0,\ldots,|x_{\lambda}|_0\}=j$ (i.e., the largest number of 0 bits is $j$ while the least number of 0 bits is larger than 0), respectively. Then, we analyze one-step transition probabilities from $x$ for both $\{\xi'_t\}^{+\infty}_{t=0}$ (i.e., without noise) and $\{\xi_t\}^{+\infty}_{t=0}$ (i.e., with noise).

For $\{\xi'_t\}^{+\infty}_{t=0}$, because only the optimal solution or the solution with the largest number of 0 bit among the parent solution and $\lambda$ offspring solutions will be accepted, we have
\begin{equation}
\begin{aligned}\label{onestep1}
&P^t_{\xi'}(x,\mathcal{X}_0)=P(0);&& \forall\; 1\leq j \leq k-1: P^t_{\xi'}(x,\mathcal{X}_j)=0;\\
&P^t_{\xi'}(x,\mathcal{X}_k)=\sum\nolimits^{k}_{j=1}P(j); &&
\forall\; k+1 \leq j \leq n: P^t_{\xi'}(x,\mathcal{X}_j)=P(j).
\end{aligned}
\end{equation}

For $\{\xi_t\}^{+\infty}_{t=0}$ with additive noise, since $\delta_2-\delta_1 \leq 2n$, we have
\begin{equation}
\begin{aligned}
&f^N(1^n) \geq f(1^n)+\delta_1 \geq 2n+\delta_2-2n=\delta_2;\\
&\forall y\neq 1^n, f^N(y)\leq f(y)+\delta_2 \leq \delta_2.
\end{aligned}
\end{equation}
For multiplicative noise, since $\delta_2>\delta_1 >0$, then
\begin{equation}
\begin{aligned}
&f^N(1^n) >0; && \forall y\neq 1^n, f^N(y) \leq 0.
\end{aligned}
\end{equation}
Thus, for these two noises, we have $\forall y \neq 1^n, f^N(1^n) \geq f^N(y)$, which implies that if the optimal solution $1^n$ is generated, it will always be accepted. Thus, we have, note that $\mathcal{X}_0=\{1^n\}$,
\begin{equation}
\begin{aligned}\label{onestep2}
&P^t_{\xi}(x,\mathcal{X}_0)=P(0).
\end{aligned}
\end{equation}
Due to the fitness evaluation disturbed by noise, the solution with the largest number of 0 bit among the parent solution and $\lambda$ offspring solutions may be rejected. Thus, we have
\begin{equation}
\begin{aligned}\label{onestep3}
&\forall \; k+1 \leq i \leq n: \sum^{n}_{j=i}P^t_{\xi}(x,\mathcal{X}_j) \leq \sum^{n}_{j=i} P(j).
\end{aligned}
\end{equation}

By combining Eq.\refeq{onestep1}, Eq.\refeq{onestep2} and Eq.\refeq{onestep3}, we have
\begin{equation}
\begin{aligned}
&\forall \; 1\leq i \leq n: \sum^{n}_{j=i}P^t_{\xi}(x,\mathcal{X}_j)\leq \sum^{n}_{j=i}P^t_{\xi'}(x,\mathcal{X}_j).
\end{aligned}
\end{equation}
Since $\sum^{n}_{j=0}P^t_{\xi}(x,\mathcal{X}_j)= \sum^{n}_{j=0}P^t_{\xi'}(x,\mathcal{X}_j)=1$, the above inequality is equivalent to
\begin{equation}
\begin{aligned}
&\forall \; 0\leq i \leq n-1: \sum^{i}_{j=0}P^t_{\xi}(x,\mathcal{X}_j)\geq \sum^{i}_{j=0}P^t_{\xi'}(x,\mathcal{X}_j),
\end{aligned}
\end{equation}
which implies that the condition Eq.\ref{analysis_condition} of Theorem \ref{analysis_approach} holds. Thus, we can get that I$_{hardest}$ problem becomes easier for (1+$\lambda$)-EA under these two kinds of noise.
\end{proof}


Theorem \ref{analysis_approach} gives a sufficient condition for that noise makes optimization easier. If its condition Eq.\ref{analysis_condition} changes the inequality direction, which implies that noise leads to a smaller probability of jumping to good states, it obviously becomes a sufficient condition for that noise makes optimization harder. We show it in Theorem \ref{analysis_approach_harmful}, the proof of which is as similar as that of Theorem \ref{analysis_approach}, except that the inequality direction needs to be changed.

\begin{theorem}\label{analysis_approach_harmful}
Given an EA $\mathcal{A}$ and a problem $f$, let a Markov chain $\{\xi_t\}^{+\infty}_{t=0}$ and a homogeneous Markov chain $\{\xi'_t\}^{+\infty}_{t=0}$ model $\mathcal{A}$ running on $f$ with noise and without noise respectively, and denote $\{\mathcal{X}_0,\mathcal{X}_1,\ldots,\mathcal{X}_m\}$ as the EFHT-Partition of $\{\xi'_t\}^{+\infty}_{t=0}$, if for all $t\geq 0$, $x \in \mathcal{X}-\mathcal{X}_0$, and for all integers $i\in [0,m-1]$,
\begin{equation}
\begin{aligned}\label{analysis_condition_harmful}
&\sum\nolimits^i_{j=0}P^t_{\xi}(x,\mathcal{X}_j) \leq \sum\nolimits^{i}_{j=0} P^t_{\xi'}(x,\mathcal{X}_j),
\end{aligned}
\end{equation}
then noise makes $f$ harder for $\mathcal{A}$, i.e., for all $x \in \mathcal{X}$, $$\expect{\tau | \xi_{0}=x} \geq \expect{\tau' | \xi'_{0}=x}.$$
\end{theorem}


Then we apply this condition to the case that (1+$\lambda$)-EA is used for optimizing the easiest case I$_{easiest}$ in the pseudo-Boolean function class. Let $\{\xi_t\}^{+\infty}_{t=0}$ and $\{\xi'_t\}^{+\infty}_{t=0}$ model (1+$\lambda$)-EA with and without noise for maximizing I$_{easiest}$ problem, respectively. It is not hard to see that the EFHT $\expect{\tau'|\xi'_0=x}$ only depends on $|x|_0$. We denote $\mathbb{E}_2(j)$ as $\expect{\tau'|\xi'_0=x}$ with $|x|_0=j$. The order of $\mathbb{E}_2(j)$ is showed in Lemma \ref{CFHT_OneMax}, the proof of which is in the Appendix.

\begin{lemma}\label{CFHT_OneMax}
For any mutation probability $0<p<0.5$, it holds that $\mathbb{E}_2(0)<\mathbb{E}_2(1)<\mathbb{E}_2(2)<\ldots<\mathbb{E}_2(n).$
\end{lemma}

\begin{theorem}\label{theo_harmful_case}
Any noise makes I$_{easiest}$ problem harder for (1+$\lambda$)-EA with mutation probability less than 0.5.
\end{theorem}
\begin{proof}
We use Theorem \ref{analysis_approach_harmful} to prove it. By Lemma \ref{CFHT_OneMax}, the EFHT-Partition of $\{\xi'_t\}^{+\infty}_{t=0}$ is $\mathcal{X}_i=\{x \in \{0,1\}^n | |x|_0=i\} \;(0\leq i\leq n)$.

For any non-optimal solution $x \in \mathcal{X}_k \;(k>0)$, we denote $P(j)\;(0 \leq j \leq n)$ as the probability that the least number of 0 bits for the $\lambda$ offspring solutions generated by bit-wise mutation on $x$ is $j$. For $\{\xi'_t\}^{+\infty}_{t=0}$, because the solution with the least number of 0 bits among the parent solution and $\lambda$ offspring solutions will be accepted, we have
\begin{equation}
\begin{aligned}
&\forall\; 0\leq j \leq k-1: P^t_{\xi'}(x,\mathcal{X}_j)=P(j); && P^t_{\xi'}(x,\mathcal{X}_k)=\sum\nolimits^{n}_{j=k}P(j); && \forall\; k+1 \leq j \leq n: P^t_{\xi'}(x,\mathcal{X}_j)=0.
\end{aligned}
\end{equation}
For $\{\xi_t\}^{+\infty}_{t=0}$, due to the fitness evaluation disturbed by noise, the solution with the least number of 0 bits among the parent solution and $\lambda$ offspring solutions may be rejected. Thus, we have
\begin{equation}
\begin{aligned}
& 0\leq i \leq k-1: \sum^{i}_{j=0}P^t_{\xi}(x,\mathcal{X}_j)\leq \sum^{i}_{j=0}P(j).
\end{aligned}
\end{equation}

Then, we can get
\begin{equation}
\begin{aligned}
&\forall \; 0\leq i \leq n-1: \sum^{i}_{j=0}P^t_{\xi}(x,\mathcal{X}_j)\leq \sum^{i}_{j=0}P^t_{\xi'}(x,\mathcal{X}_j).
\end{aligned}
\end{equation}
This implies that the condition Eq.\refeq{analysis_condition_harmful} of Theorem \ref{analysis_approach_harmful} holds. Thus, by Theorem \ref{analysis_approach_harmful}, we can get that noise makes I$_{easiest}$ problem harder for (1+$\lambda$)-EA.
\end{proof}

\subsection{Discussion}

We have shown that noise makes I$_{hardest}$ and I$_{easiest}$ problems easier and harder, respectively, for (1+$\lambda$)-EA. These two problems are known to be the hardest and the easiest instance respectively in the pseudo-Boolean function class with a unique global optimum for the (1+1)-EA \cite{qian2012algorithm}. We can intuitively interpret the discovered effect of noise for EAs on these two problems. For I$_{hardest}$ problem, the EA searches along the deceptive direction while noise can add some randomness to make the EA have some possibility to run along the right direction; for I$_{easiest}$ problem, the EA searches along the right direction while noise can only harm the optimization process. We thus hypothesize that we need to take care of the noise only when the optimization problem is moderately or less complex.

To further verify our hypothesis, we employ the Jump$_{m,n}$ problem, which is a problem with adjustable difficulty and can be configured as I$_{eaisest}$ when $m=1$ and I$_{hardest}$ when $m=n$.
\begin{definition}[Jump$_{m,n}$ Problem]\label{def_jump_mn}
    Jump$_{m,n}$ Problem of size $n$ with $1 \leq m \leq n$ is to find an $n$ bits binary string $x^*$ such that
    \begin{aligna}
    &x^* =\arg\max\nolimits_{x \in \{0,1\}^n}\bigg(\text{Jump$_{m,n}(x)$}=
    \begin{cases}
    m+\sum^n_{i=1} x_i & \text{if $\sum^n_{i=1} x_i \leq n-m$ or $\sum^n_{i=1} x_i=n$}\\
    n-\sum^n_{i=1} x_i & \text{otherwise}
    \end{cases}\bigg),
    \end{aligna}
    where $x_i$ is the $i$-th bit of a solution $x \in \{0,1\}^n$.
\end{definition}
We test (1+1)-EA with mutation probability $\frac{1}{n}$ on Jump$_{m,n}$. It is known that the expected running time of the (1+1)-EA on Jump$_{m,n}$ is $\Theta(n^m+n \log n)$ \cite{droste2002analysis}, which implies that Jump$_{m,n}$ with larger value of $m$ is harder.
In the experiment, we set $n=5$, and for noise, we use the additive noise with $\delta_1=-0.5n \wedge \delta_2=0.5n$, the multiplicative noise with $\delta_1=1 \wedge \delta_2=2$, and the one-bit noise with $p_n=0.5$, respectively. We record the expected running time gap starting from each initial solution
$$
gap=(\expect{\tau}-\expect{\tau'})/\expect{\tau'},
$$
where $\expect{\tau}$ and $\expect{\tau'}$ denote the expected running time of the EA optimizing the problem with and without noise, respectively. The larger the gap means that the noise has a more negative effect, while the smaller the gap means that the noise has a less negative effect. For each initial solution and each configuration of noise, we repeat the running of the (1+1)-EA 1000 times, and estimate the expected running time by the average running time, and thus estimate the gap. The results are plotted in Figure \ref{fig_ratio}.

\begin{figure*}[h!]\centering
\begin{minipage}[c]{0.33\linewidth}\centering
        \includegraphics[width=0.8\linewidth,height=0.65\linewidth]{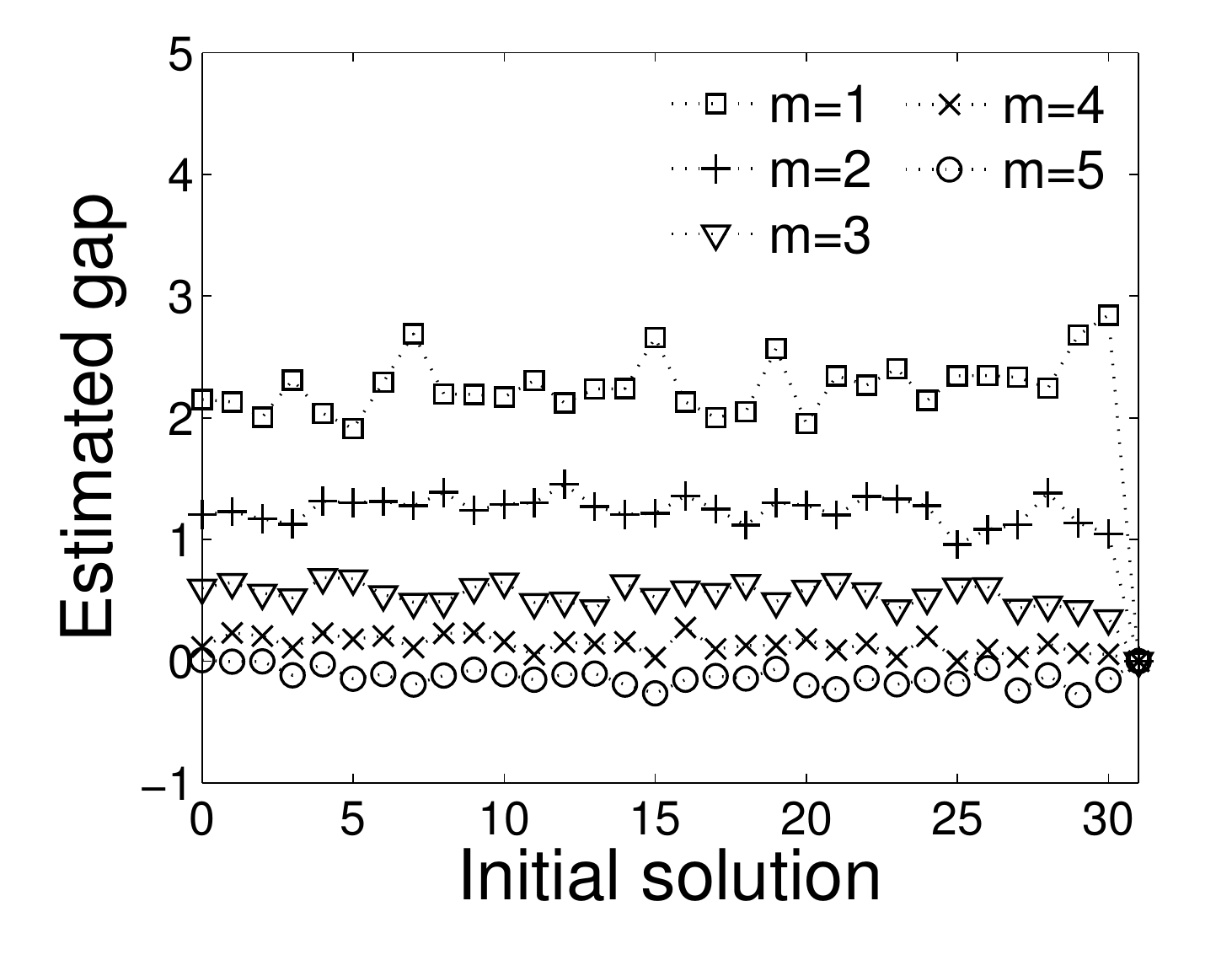}
\end{minipage}
\begin{minipage}[c]{0.33\linewidth}\centering
        \includegraphics[width=0.8\linewidth,height=0.65\linewidth]{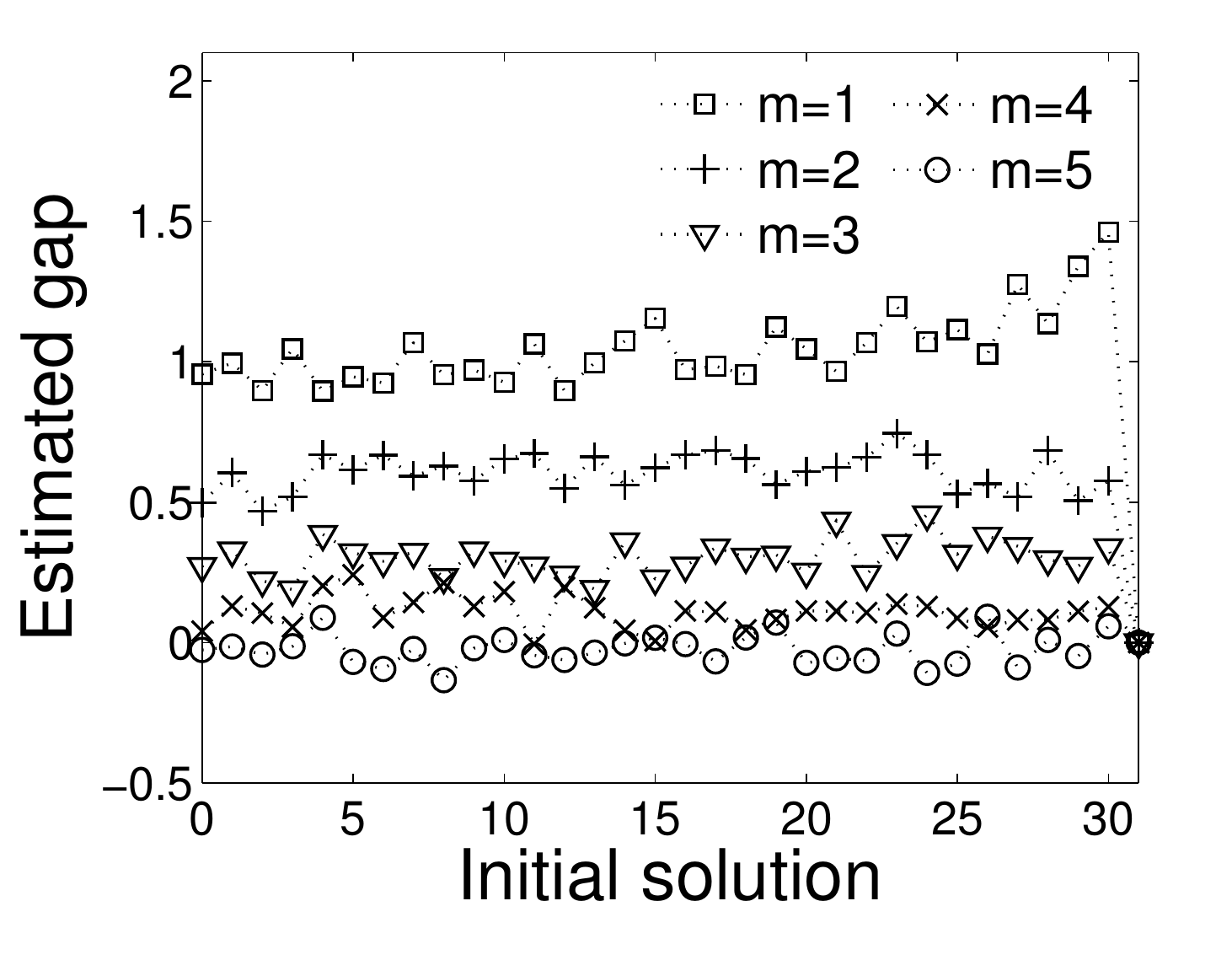}
\end{minipage}
\begin{minipage}[c]{0.33\linewidth}\centering
        \includegraphics[width=0.8\linewidth,height=0.65\linewidth]{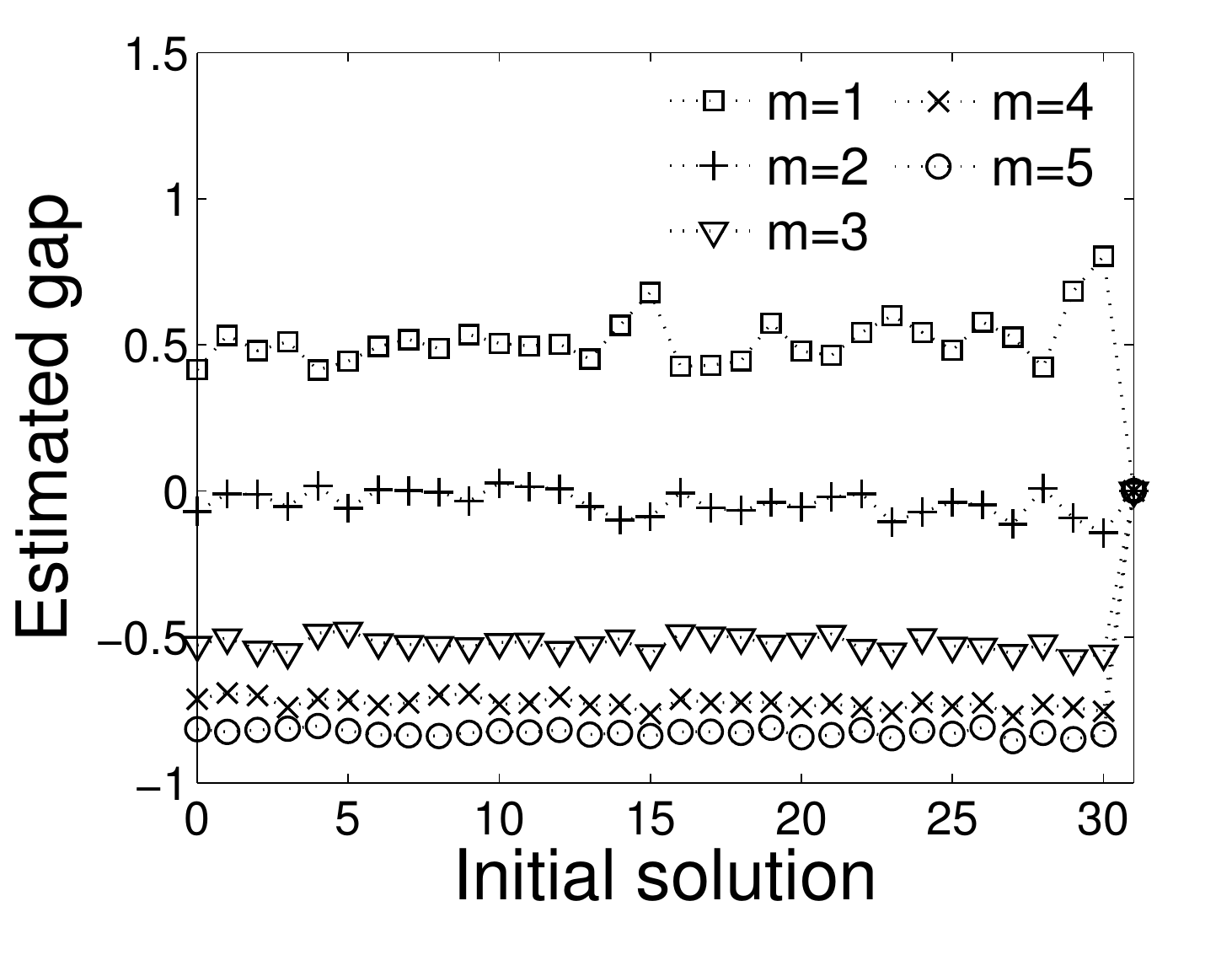}
\end{minipage}\\
\begin{minipage}[c]{0.33\linewidth}\centering
    \small(a) additive noise
\end{minipage}
\begin{minipage}[c]{0.33\linewidth}\centering
    \small(b) multiplicative noise
\end{minipage}
\begin{minipage}[c]{0.33\linewidth}\centering
    \small(c) one-bit noise
\end{minipage}\\\vspace{-0.6em}
\caption{Estimated ERT gap for (1+1)-EA solving Jump$_{m,5}$ problem with or without noise.}\label{fig_ratio}
\end{figure*}

We can observe that the gaps for larger $m$ are lower (i.e., the negative effect by noise decreases as the problem hardness increases), and the gaps for large $m$ tend to be 0 or negative values (i.e., noise can have no or positive effect when the optimization is quite hard). These empirical observations give support to our hypothesis that the noise should be handled carefully only when the optimization is moderately or less complex.

\section{On the Usefulness of Noise Handling Strategies}

\subsection{Re-evaluation}

There are naturally two fitness evaluation options for EAs \cite{arnold2002local,jin2005evolutionary,goh2007investigation,heidrich2009hoeffding}:
\begin{itemize}
\item \textbf{single-evaluation} we evaluate a solution once, and use the evaluated fitness for this solution in the future.
\item \textbf{re-evaluation} every time we access the fitness of a solution by evaluation.
\end{itemize}
For example, for (1+1)-EA in Algorithm \ref{(1+1)-EA}, if using re-evaluation, both $f(x')$ and $f(x)$ will be calculated and recalculated in each iteration; if using single-evaluation, only $f(x')$ will be calculated and the previous obtained fitness $f(x)$ will be reused. Intuitively, re-evaluation can smooth noise and thus could be better for noisy optimizations, but it also increases the fitness evaluation cost and thus increases the running time. Its usefulness was not yet clear. Note that, the analysis in the previous section assumes single-evaluation.

In this section, we take the I$_{easiest}$ problem, where noise has been proved to have a strong negative effect in the previous section, as the representative problem, and compare these two options for (1+1)-EA with mutation probability $\frac{1}{n}$ solving this problem under one-bit noise to show whether re-evaluation is useful. Note that for one-bit noise, $p_n$ controls the noise level, that is, noise becomes stronger as $p_n$ gets larger, and it is also the variable of the PNT.

\begin{theorem}\label{runtime_without}
The PNT of (1+1)-EA using single-evaluation with mutation probability $\frac{1}{n}$ on I$_{easiest}$ problem is lower bounded by $1-1/\Omega(poly(n))$ and upper bounded by $1-1/O(2^npoly(n))$, where $poly(n)$ indicates any polynomial of $n$, with respect to one-bit noise.
\end{theorem}

The theorem is straightforwardly derived from the following lemma.

\begin{lemma}\label{runtime_single}
For (1+1)-EA using single-evaluation with mutation probability $\frac{1}{n}$ on I$_{easiest}$ problem under one-bit noise, the expected running time is $O(n^2+n/(1-p_n))$ and $\Omega(np_n/(2^n(1-p_n)))$.
\end{lemma}
\begin{proof}
Let $L$ denote the noisy fitness value $f^N(x)$ of the current solution $x$. Because (1+1)-EA does not accept a solution with a smaller fitness (i.e., the 4th step of Algorithm \ref{(1+1)-EA}) and it doesn't re-evaluate the fitness of the current solution $x$, $L\;(0\leq L\leq n)$ will never decrease. We first analyze the expected steps until $L$ increases when starting from $L=i$ (denoted by $\expect{i}$), and then sum up them to get an upper bound $\sum\nolimits^{n-1}_{i=0} \expect{i}$ for the expected steps until $L$ reaches the maximum value $n$. For $\expect{i}$, we analyze the probability $P$ that $L$ increases in two steps when $L=i$, then $\expect{i}=2 \cdot \frac{1}{P}$. Note that, one-bit noise can make $L$ be $|x|_1-1$, $|x|_1$ or $|x|_1+1$, where $|x|_1=\sum^{n}_{i=1} x_i$ is the number of 1 bits. When analyzing the noisy fitness $f^N(x')$ of the offspring $x'$ in each step, we need to first consider bit-wise mutation on $x$ and then one random bit flip for noise.

When $0<L<n-1$, $|x|_1=L-1$, $L$ or $L+1$.\\
(1) For $|x|_1=L-1$, $P\geq\frac{n-L+1}{n}(1-\frac{1}{n})^{(n-1)}p_n\frac{n-L}{n}+\frac{n-L+1}{n} (1-\frac{1}{n})^{(n-1)}(1-p_n)\frac{n-L}{n}(1-\frac{1}{n})^{(n-1)}(1-p_n)$,
since it is sufficient to flip one 0 bit for mutation and one 0 bit for noise in the first step, or flip one 0 bit for mutation and no bit for noise in the first step and flip one 0 bit for mutation and no bit for noise in the second step.\\
(2) For $|x|_1=L$, $P\geq(1-\frac{1}{n})^np_n\frac{n-L}{n}+\frac{n-L}{n}(1-\frac{1}{n})^{n-1}(1-p_n),$ since it is sufficient to flip no bit for mutation and one 0 bit for noise, or flip one 0 bit for mutation and no bit for noise in the first step. \\
(3) For $|x|_1=L+1$, $P\geq(1-\frac{1}{n})^{n}(1-p_n+p_n\frac{n-L-1}{n}),$ since it is sufficient to flip no bit for mutation and no bit or one 0 bit for noise in the first step. \\
Thus, for these three cases, we have
\begin{equation}
\begin{aligned}\label{one-step-probability}
P&\geq p_n(1-\frac{1}{n})^{(n-1)}\frac{n-L}{n}\frac{n-L-1}{n}+(1-\frac{1}{n})^{2(n-1)} (1-p_n)^2 \frac{n-L}{n}\frac{n-L-1}{n}\\
&\geq^1 (p_n+(1-p_n)^2) \frac{(n-L)(n-L-1)}{e^2n^2} \geq^2 \frac{3(n-L)(n-L-1)}{4e^2n^2},
\end{aligned}
\end{equation}
where the `$\geq^1$' is by $(1-\frac{1}{n})^{n-1} \geq \frac{1}{e}$ and the `$\geq^2$' is by $0\leq p_n\leq 1$.

When $L=0$, $|x|_1=0$ or 1. By considering case (2) and (3), we can get the same lower bound for $P$.

When $L=n-1$ and the optimal solution $1^n$ has not been found, $|x|_1=n-2$ or $n-1$. By considering case (1) and (2), we can get $P \geq 3/(2e^2n^2)$.

Based on the above analysis, we can get that the expected steps until $L=n$ is at most
$$\sum\nolimits^{n-1}_{i=0} \expect{i}\leq 2 \cdot (\sum^{n-2}_{L=0}\frac{4e^2n^2}{3(n-L)(n-L-1)}+\frac{2e^2n^2}{3}),\; \text{i.e.}, O(n^2).$$

When $L=n$, $|x|_1=n-1$ or $n$ (i.e., the optimal solution has been found). If $|x|_1=n-1$, the optimal solution will be generated and accepted in one step with probability $\frac{1}{n}(1-\frac{1}{n})^{n-1}(1-p_n)\geq \frac{(1-p_n)}{en}$, because it needs to flip the unique 0 bit for mutation and no bit for noise. This implies that the expected steps for finding the optimal solution is at most $\frac{en}{(1-p_n)}$.

Thus, we can get the upper bound $O(n^2+\frac{n}{1-p_n})$ for the expected running time of the whole process.

Then, we are to analyze the lower bound. Assume that the initial solution $x_{init}$ has $n-1$ number of 1 bits, i.e., $|x_{init}|_1=n-1$. If the fitness of $x_{init}$ is evaluated as $n$, which happens with probability $p_n\frac{1}{n}$, before finding the optimal solution, the solution will always have $n-1$ number of 1 bits and its fitness will always be $n$. From the above analysis, we know that in such a situation, the probability of generating and accepting the optimal solution in one step is $\frac{1}{n}(1-\frac{1}{n})^{n-1}(1-p_n) \leq \frac{(1-p_n)}{n}$. Thus, the expected running time for finding the optimal solution when starting from $|x_{init}|_1=n-1$ is at least $p_n\frac{1}{n} \cdot \frac{n}{(1-p_n)}=\frac{p_n}{(1-p_n)}$. Because the initial solution is uniformly distributed over $\{0,1\}^n$, the probability that the algorithm starts from $|x_{init}|_1=n-1$ is $n/2^n$. Thus, we can get the lower bound $\Omega(\frac{np_n}{2^n(1-p_n)})$ for the expected running time of the whole process.
\end{proof}

\begin{theorem}\label{runtime_with}
The PNT of (1+1)-EA using re-evaluation with mutation probability $\frac{1}{n}$ on I$_{easiest}$ problem is $\Theta(\frac{\log(n)}{n})$, with respect to one-bit noise.
\end{theorem}

The theorem is straightforwardly derived from the following lemma.

\begin{lemma} [\cite{droste2004analysis}]
For (1+1)-EA using re-evaluation with mutation probability $\frac{1}{n}$ on I$_{easiest}$ problem under one-bit noise, the expected running time is polynomial when $p_n\in O(\log(n)/n)$, and the running time is polynomial with super-polynomially small probability when $p_n\in \omega(\log(n)/n)$.
\end{lemma}

\subsection{Threshold Selection}

During the process of evolutionary optimization, most of the improvements in one generation are small. When using re-evaluation, due to noisy fitness evaluation, a considerable portion of these improvements are not real, where a worse solution appears to have a ``better" fitness and then survives to replace the true better solution which has a ``worse" fitness. This may mislead the search direction of EAs, and then slow down the efficiency of EAs or make EAs get trapped in the local optimal solution, as observed in Section 4.1. To deal with this problem, a selection strategy for EAs handling noise was proposed \cite{markon2001thresholding}.
\begin{itemize}
\item \textbf{threshold selection} an offspring solution will be accepted only if its fitness is larger than the parent solution by at least a predefined threshold $\tau \geq 0$.
\end{itemize}
For example, for (1+1)-EA with threshold selection as in Algorithm \ref{(1+1)-EA-threshold}, its 4th step changes to be ``if {$f(x') \geq f(x)+\tau$}" rather than ``if {$f(x') \geq f(x)$}" in Algorithm \ref{(1+1)-EA}. Such a strategy can reduce the risk of accepting a bad solution due to noise. Although the good local performance (i.e., the progress of one step) of EAs with threshold selection has been shown on some problems \cite{markon2001thresholding,beielstein2002threshold,bartz2005new}, its usefulness for the global performance (i.e., the running time until finding the optimal solution) of EAs under noise is not yet clear.

\begin{algorithm}[(1+1)-EA with threshold selection]\label{(1+1)-EA-threshold} Given pseudo-Boolean function $f$ with solution length $n$, and a predefined threshold $\tau \geq 0$, it consists of the following steps:\\
    \begin{tabular}{ll}
    1. & $x:=$ randomly selected from $\{0,1\}^{n}$.\\
    2. & Repeat until the termination condition is met\\
    3. & \quad $x':=$ flip each bit of $x$ with probability $p$. \\
    4. &\quad if {$f(x') \geq f(x)+\tau$} \\
    5. &\quad \quad $x:=x'$.
    \end{tabular}\\
where $p \in (0,0.5)$ is the mutation probability.
\end{algorithm}

In this section, we compare the running time of (1+1)-EA with and without threshold selection solving I$_{easiest}$ problem under one-bit noise to show whether threshold selection will be useful. Note that, the analysis here assumes re-evaluation.

Algorithm \ref{random_walk} shows a random walk on a graph. Lemma \ref{theo_randwalk} gives an upper bound on the expected steps for a random walk to visit each vertex of a graph at least once, which will be used in the following analysis.

\begin{algorithm}[Random Walk]\label{random_walk}
Given an undirected connected graph $G=(V,E)$ with vertex set $V$ and edge set $E$, it consists of the following steps:\\
    \begin{tabular}{ll}
    1. & start at a vertex $v \in V$.\\
    2. & Repeat until the termination condition is met\\
    3. &\quad choose a neighbor $u$ of $v$ in $G$ uniformly at random. \\
    4. &\quad set $v:=u$.
    \end{tabular}\\
\end{algorithm}

\begin{lemma}[\cite{aleliunas1979random}]\label{theo_randwalk}
Given an undirected connected graph $G=(V,E)$, the expected cover time of a random walk on $G$ is upper bounded by $2|E|(|V|-1)$, where the cover time of a random walk on $G$ is the number of steps until each vertex $v \in V$ has been visited at least once.
\end{lemma}

\begin{theorem}\label{runtime_theo_threshold}
The PNT of (1+1)-EA using re-evaluation with threshold selection $\tau=1$ and mutation probability $\frac{1}{n}$ on I$_{easiest}$ problem is not less than $\frac{1}{2e}$, with respect to one-bit noise.
\end{theorem}

The theorem can be directly derived from the following lemma.

\begin{lemma}\label{theo_threshold}
For (1+1)-EA using re-evaluation with threshold selection $\tau=1$ and mutation probability $\frac{1}{n}$ on I$_{easiest}$ problem under one-bit noise, the expected running time is $O(n^3)$ when $p_n \leq \frac{1}{2e}$.
\end{lemma}
\begin{proof}
We denote the number of one bits of the current solution $x$ by $L\;(0\leq L \leq n)$. Let $P_d$ denote the probability that the offspring solution $x'$ by bit-wise mutation on $x$ has $L+d \;(-L\leq d \leq n-L)$ number of one bits, and let $P'_d$ denote the probability that the next solution after bit-wise mutation and selection has $L+d$ number of one bits.

Then, we analyze $P'_d$. We consider $0 \leq L \leq n-1$. Note that one-bit noise can change the true fitness of a solution by at most 1, i.e., $|f^N(x)-f(x)|\leq 1$.\\
(1) When $d \leq -2$, $f^N(x') \leq L+d+1 \leq L-1 \leq f^N(x)$. Because an offspring solution will be accepted only if $f^N(x') \geq f^N(x)+1$, the offspring solution $x'$ will be discarded in this case, which implies that $\forall d \leq -2: P'_d=0$.\\
(2) When $d=-1$, the offspring solution $x'$ will be accepted only if $f^N(x')=L \wedge f^N(x)=L-1$, the probability of which is $p_n\frac{n-L+1}{n}\cdot p_n\frac{L}{n}$, since it needs to flip one 0 bit of $x'$ and flip one 1 bit of $x$. Thus, $P'_{-1}=P_{-1}\cdot(p_n\frac{L}{n}p_n\frac{n-L+1}{n})$.\\
(3) When $d=1$, if $f^N(x)=L-1$, the probability of which is $p_n\frac{L}{n}$, the offspring solution $x'$ will be accepted, since $f^N(x') \geq L+1-1=L>f^N(x)$; if $f^N(x)=L \wedge f^N(x')\geq L+1$, the probability of which is $(1-p_n)\cdot(1-p_n+p_n\frac{n-L-1}{n})$, $x'$ will be accepted; if $f^N(x)=L+1 \wedge f^N(x')=L+2$, the probability of which is $p_n\frac{n-L}{n}\cdot p_n\frac{n-L-1}{n}$, $x'$ will be accepted; otherwise, $x'$ will be discarded. Thus, $P'_{1}=P_{1}\cdot(p_n\frac{L}{n}+(1-p_n)(1-p_n+p_n\frac{n-L-1}{n})+p_n\frac{n-L}{n}p_n\frac{n-L-1}{n})$.\\
(4) When $d \geq 2$, it is easy to see that $P'_d>0$.

Because we are to get the upper bound of the expected running time for finding the optimal solution $1^n$ for the first time, we pessimistically assume that $\forall d \geq 2: P'_d=0$. Then, we compare $P'_1$ with $P'_{-1}$.
\begin{equation}
\begin{aligned}
& P'_1\geq P_1 p_n \frac{L}{n} \geq \frac{n-L}{n}(1-\frac{1}{n})^{n-1}p_n\frac{L}{n}\geq p_n\frac{L(n-L)}{en^2},
\end{aligned}
\end{equation}
where the second inequality is by $P_1 \geq \frac{n-L}{n}(1-\frac{1}{n})^{n-1}$ since it is sufficient to flip just one 0 bit, and the last inequality is by $(1-\frac{1}{n})^{n-1}\geq \frac{1}{e}$.
\begin{equation}
\begin{aligned}
& P'_{-1}=P_{-1}(p_n\frac{L}{n}p_n\frac{n-L+1}{n})\leq\frac{L}{n}(p_n\frac{L}{n}p_n\frac{n-L+1}{n}) \leq p_n\frac{L}{en^2} \cdot \frac{L(n-L+1)}{2n}\leq p_n\frac{L(n-L)}{en^2},
\end{aligned}
\end{equation}
where the first inequality is by $P_{-1}\leq \frac{L}{n}$ since it is necessary to flip at least one 1 bit, the second inequality is by $p_n\leq \frac{1}{2e}$, and the last inequality is by $\frac{L(n-L+1)}{2n}\leq n-L$.

Thus, we have for all $0\leq L\leq n-1$, $P'_{1} \geq P'_{-1}$. Because we are to get the upper bound of the expected running time for finding $1^n$, we can pessimistically assume that $P'_{1} = P'_{-1}$. Then, we can view the evolutionary process as a random walk on the path $\{0,1,2,\ldots,n\}$. We call a step that jumps to the neighbor state a relevant step. Thus, by Lemma \ref{theo_randwalk}, it needs at most $2n^2$ expected relevant steps to find $1^n$. Because the probability of a relevant step is at least $P'_{1} \geq P_1(1-p_n)^2\geq \frac{n-L}{n}(1-\frac{1}{n})^{n-1}(1-\frac{1}{2e})^2 \geq (1-\frac{1}{2e})^2/en$, the expected running time for a relevant step is $O(n)$. Thus, the expected running time of (1+1)-EA with $\tau=1$ on I$_{easiest}$ problem with $p_n \leq \frac{1}{2e}$ is upper bounded by $O(n^3)$.
\end{proof}

\begin{theorem}\label{runtime_theo_threshold2}
The PNT of (1+1)-EA using re-evaluation with threshold selection $\tau=2$ and mutation probability $\frac{1}{n}$ on I$_{easiest}$ problem is lower bounded by$1-1/\Omega(poly(n))$ and upper bounded by $1-1/O(2^npoly(n))$, where $poly(n)$ indicates any polynomial of $n$, with respect to one-bit noise.
\end{theorem}

The theorem can be directly derived from the following lemma.

\begin{lemma}\label{theo_threshold2}
For (1+1)-EA using re-evaluation with threshold selection $\tau=2$ and mutation probability $\frac{1}{n}$ on I$_{easiest}$ problem under one-bit noise, the expected running time is $O(n\log n/(p_n(1-p_n)))$ and \\ $\Omega(n^2/(2^np_n(1-p_n)))$.
\end{lemma}
\begin{proof}
Let $L\;(0\leq L \leq n)$ denote the number of one bits of the current solution $x$. Here, an offspring solution $x'$ will be accepted only if $f^N(x')-f^N(x) \geq 2$. As in the proof of Lemma \ref{theo_threshold}, we can derive \begin{equation}
\begin{aligned}
&\forall d \leq -1: P'_d=0; \\
&P'_1=P_{1}\big(p_n\frac{L}{n}((1-p_n)+p_n\frac{n-L-1}{n})+(1-p_n)(p_n\frac{n-L-1}{n})\big);\\
&\forall d \geq 2: P'_d>0.
\end{aligned}
\end{equation}
Thus, $L$ will never decrease in the evolution process, and it can increase in one step with probability
\begin{equation}
\begin{aligned}
P'_{d>0} &> P'_1 \geq \frac{n-L}{n}(1-\frac{1}{n})^{(n-1)}((1-p_n)p_n(1-\frac{1}{n})+p^2_n\frac{L(n-L-1)}{n^2})\\
& \geq \frac{1}{2e} (1-p_n)p_n \frac{n-L}{n}.
\end{aligned}
\end{equation}
Then, we can get that the expected steps until $L=n$ (i.e., the optimal solution is found) is at most
$$
\sum^{n-1}_{L=0} \frac{2en}{(1-p_n)p_n (n-L)},\; \text{i.e.}, O(\frac{n \log n}{p_n(1-p_n)}).
$$

Then, we are to analyze the lower bound. Assume that the initial solution $x_{init}$ has $n-1$ number of 1 bits. Before finding the optimal solution, the solution $x$ in the population will always satisfy $|x|_1=n-1$ because $\forall d \leq -1: P'_d=0$. The optimal solution (i.e., $|x|_1=n$) will be found in one step with probability $P'_1=P_1p_n(1-p_n)(1-\frac{1}{n})=\frac{1}{n}(1-\frac{1}{n})^{(n-1)}p_n(1-p_n)(1-\frac{1}{n}) \leq \frac{p_n(1-p_n)}{en}$. Thus, the expected steps for finding the optimal solution when starting from $|x_{init}|_1=n-1$ is at least $\frac{en}{p_n(1-p_n)}$. By the uniform distribution of the initial solution, the probability that $|x_{init}|_1=n-1$ is $n/2^n$. Thus, we can get the lower bound $\Omega(\frac{n^2}{2^np_n(1-p_n)})$ for the expected running time of the whole process.
\end{proof}



\subsection{Smooth Threshold Selection}

We propose the smooth threshold selection as in Definition \ref{smooth}, which modifies the original threshold selection by changing the hard threshold value to a smooth one. We are to show that, by such a small modification, the PNT of (1+1)-EA on I$_{easiest}$ problem is improved to 1, which means that the expected running time of (1+1)-EA is always polynomial disregard the one-bit noise level.

\begin{definition}[Smooth Threshold Selection]\label{smooth}
Let $\delta$ be the gap between the fitness of the offspring solution $x'$ and the parent solution $x$, i.e., $\delta=f(x')-f(x)$. Then, the selection process will behave as follows:\\
(1) if $\delta \leq 0$, $x'$ will be rejected;\\
(2) if $\delta=1$, $x'$ will be accepted with probability $\frac{1}{5n}$;\\
(3) if $\delta>1$, $x'$ will be accepted.
\end{definition}

\begin{theorem}\label{theo_threshold_smart}
The PNT of (1+1)-EA using re-evaluation with smooth threshold selection and mutation probability $\frac{1}{n}$ on I$_{easiest}$ problem is 1, with respect to one-bit noise.
\end{theorem}
\begin{proof}
We first analyze $P'_{d}$ as that analyzed in the proof of Lemma \ref{theo_threshold}. The only difference is that when the fitness gap between the offspring and the parent solution is 1, the offspring solution will be accepted with probability $\frac{1}{5n}$ here, while it will be always accepted in the proof of Lemma \ref{theo_threshold}. Thus, for smooth threshold selection, we can similarly derive
\begin{equation}
\begin{aligned}
&\forall d \leq -2: P'_d=0; \\
&P'_{-1}=P_{-1}(p_n\frac{L}{n}p_n\frac{n-L+1}{n}) \cdot \frac{1}{5n};\\
&P'_1=P_{1}\big(p_n\frac{L}{n}(p_n\frac{L+1}{n} \cdot \frac{1}{5n}+(1-p_n)+p_n\frac{n-L-1}{n})+(1-p_n)((1-p_n)\cdot \frac{1}{5n}+p_n\frac{n-L-1}{n})\\
&\qquad +p_n\frac{n-L}{n}p_n\frac{n-L-1}{n}\cdot \frac{1}{5n}\big);\\
& \forall d \geq 2: P'_d>0.
\end{aligned}
\end{equation}

Note that $L$ ($0 \leq L \leq n$) denotes the number of one bits of the current solution $x$. Our goal is to reach $L=n$. If starting from $L=n-1$, $L$ will reach $n$ in one step with probability
\begin{equation}
\begin{aligned}
&P'_1 \geq P_1 (p_n\frac{L}{n} p_n\frac{L+1}{n} \cdot \frac{1}{5n}+(1-p_n)(1-p_n)\cdot \frac{1}{5n})\\
& \geq \frac{n-L}{n}(1-\frac{1}{n})^{n-1} (p_n\frac{L}{n} p_n\frac{L+1}{n} \cdot \frac{1}{5n}+(1-p_n)(1-p_n)\cdot \frac{1}{5n})\\
& \geq \frac{1}{5en^2}(\frac{n-1}{n}p_n^2+(1-p_n)^2) \quad (\text{by $L=n-1$ and $(1-\frac{1}{n})^{n-1} \geq \frac{1}{e}$})\\
& \geq \frac{1}{5en^2} \cdot \frac{n-1}{2n-1} \in \Omega(\frac{1}{n^2}). \quad (\text{by $0 \leq p_n \leq 1$})
\end{aligned}
\end{equation}
Thus, for reaching $L=n$, we need to reach $L=n-1$ for $O(n^2)$ times in expectation.

Then, we analyze the expected running time until $L=n-1$. In this process, we can pessimistically assume that $L=n$ will never be reached, because our final goal is to get the upper bound on the expected running time for reaching $L=n$. For $0 \leq L \leq n-2$, we have
\begin{equation}
\begin{aligned}
&\frac{P'_1}{P'_{-1}} \geq \frac{P_1 \cdot (p_n\frac{L}{n} p_n\frac{n-L-1}{n})}{P_{-1}\cdot (p_n\frac{L}{n}p_n\frac{n-L+1}{n}) \cdot \frac{1}{5n}} \geq \frac{\frac{n-L}{n}(1-\frac{1}{n})^{n-1} \cdot (p_n\frac{L}{n} p_n\frac{n-L-1}{n})}{\frac{L}{n}\cdot (p_n\frac{L}{n}p_n\frac{n-L+1}{n}) \cdot \frac{1}{5n}}\\
&\geq \frac{5n(n-L)(n-L-1)}{eL(n-L+1)}=\frac{5n(\frac{n}{L}-1)}{e(1+\frac{2}{n-L-1})} > 1.
\end{aligned}
\end{equation}
Again, we can pessimistically assume that $P'_1=P'_{-1}$ and $\forall d \geq 2, P'_{d}=0$, because we are to get the upper bound on the expected running time until $L=n-1$. Then, we can view the evolutionary process for reaching $L=n-1$ as a random walk on the path $\{0,1,2,\ldots,n-1\}$. We call a step that jumps to the neighbor state a relevant step. Thus, by Lemma \ref{theo_randwalk}, it needs at most $2(n-1)^2$ expected relevant steps to reach $L=n-1$. Because the probability of a relevant step is at least
\begin{equation}
\begin{aligned}
& P'_{1} \geq P_1((1-p_n)(1-p_n)\cdot \frac{1}{5n}+ p_n\frac{n-L}{n}p_n\frac{n-L-1}{n}\cdot \frac{1}{5n})\\
& \geq \frac{n-L}{5en^2}((1-p_n)^2 + p^2_n\frac{(n-L)(n-L-1)}{n^2})\\
& \geq \frac{2}{5en^2}((1-p_n)^2+\frac{2}{n^2}p^2_n) \geq \frac{2}{5en^2} \cdot \frac{2}{n^2+2},
\end{aligned}
\end{equation}
the expected running time for a relevant step is $O(n^4)$. Then, the expected running time for reaching $L=n-1$ is $O(n^6)$.

Thus, the expected running time of the whole optimization process is $O(n^8)$ for any $p_n \in [0,1]$, and then this theorem holds.
\end{proof}


We draw an intuitive understanding from the proof of Theorem \ref{theo_threshold_smart} that why the smooth threshold selection can be better than the original threshold selections. By changing the hard threshold to be a smooth threshold, it can not only make the probability of accepting a false better solution in one step small enough, i.e. $P'_1 \geq P'_{-1}$, but also make the probability of producing progress in one step large enough, i.e., $P'_1$ is not small.

\section{Discussions and Conclusions}

This paper studies theoretical issues of noisy optimization by evolutionary algorithms. 

First, we discover that an optimization problem may become easier instead of harder in a noisy environment. We then derive a sufficient condition under which noise makes optimization easier or harder. By filling this condition, we have shown that for (1+$\lambda$)-EA, noise makes the optimization on the hardest and the easiest case in the pseudo-Boolean function class easier and harder, respectively. We also hypothesize that we need to take care of noise only when the optimization problem is moderately or less complex. Experiments on the Jump$_{m,n}$ problem, which has an adjustable difficulty parameter, supported our hypothesis.

In problems where the noise has a negative effect, we then study the usefulness of two commonly employed noise-handling strategies, re-evaluation and threshold selection. The study takes the easiest case in the pseudo-Boolean function class as the representative problem, where the noise significantly harms the expected running time of the (1+1)-EA. We use the polynomial noise tolerance (PNT) level as the performance measure, and analyzed the PNT of each EA. 

The re-evaluation strategy seems to be a reasonable method for reducing random noise. However, we derive that the (1+1)-EA with single-evaluation has a PNT lower bound $1-1/\Omega(poly(n))$ from Theorem~5 which is close to $1$, whilst the (1+1)-EA with re-evaluation has the PNT $\Theta(\log(n)/n)$ which can be quite close to zero as $n$ is large. It is surprise to see that the re-evaluation strategy leads to a much worse noise tolerance than that without any noise handling method.

The re-evaluation with threshold selection strategy has a better PNT comparing with the re-evaluation alone. When the threshold is 1, we derive a PNT lower bound $\frac{1}{2e}$ from Theorem 7, and when the threshold is 2, we obtain $1-1/\Omega(poly(n))$ from Theorem 8. The improvement from re-evaluation alone could be explained as that the threshold selection filters out fake progresses that caused by the noise. However, it still showed no improvements from the (1+1)-EA without any noise handling method.

We then proposed the smooth threshold selection, which acts like the threshold selection with threshold 2 but accepts progresses 1 with a probability. We proved that the (1+1)-EA with the smooth threshold selection has the PNT 1 from Theorem 9, which exceeds that of (1+1)-EA without any noise handling method. Our explanation is that, like the original threshold selection, the proposed one filters out fake progresses, while it also keep some chances to accept real progresses.

Although the investigated EAs and problems in this paper are simple and specifically used for the theoretical analysis of EAs, the analysis still disclosed counter-intuitive results and, particularly, demonstrated that theoretical investigation is essential in designing better noise handling strategies. We are optimistic that our findings may be helpful for practical uses of EAs, which will be studied in the future.

\section{Acknowledgements}

to be added ...

\bibliography{ectheory}

\begin{thebibliography}{28}
\providecommand{\natexlab}[1]{#1}
\providecommand{\url}[1]{\texttt{#1}}
\expandafter\ifx\csname urlstyle\endcsname\relax
  \providecommand{\doi}[1]{doi: #1}\else
  \providecommand{\doi}{doi: \begingroup \urlstyle{rm}\Url}\fi

\bibitem[Aleliunas et~al.(1979)Aleliunas, Karp, Lipton, Lovasz, and
  Rackoff]{aleliunas1979random}
R.~Aleliunas, R.~Karp, R.~Lipton, L.~Lovasz, and C.~Rackoff.
\newblock Random walks, universal traversal sequences, and the complexity of
  maze problems.
\newblock In \emph{Proceedings of the 20th Annual Symposium on Foundations of
  Computer Science (FOCS'79)}, pages 218--223, San Juan, Puerto Rico, 1979.

\bibitem[Arnold and Beyer(2002)]{arnold2002local}
D.~V. Arnold and H.-G. Beyer.
\newblock Local performance of the (1+1)-{ES} in a noisy environment.
\newblock \emph{IEEE Transactions on Evolutionary Computation}, 6\penalty0
  (1):\penalty0 30--41, 2002.

\bibitem[Arnold and Beyer(2003)]{arnold2003comparison}
D.~V. Arnold and H.-G. Beyer.
\newblock A comparison of evolution strategies with other direct search methods
  in the presence of noise.
\newblock \emph{Computational Optimization and Applications}, 24\penalty0
  (1):\penalty0 135--159, 2003.

\bibitem[B{\"{a}}ck(1996)]{back:96}
T.~B{\"{a}}ck.
\newblock \emph{Evolutionary Algorithms in Theory and Practice: Evolution
  Strategies, Evolutionary Programming, Genetic Algorithms}.
\newblock Oxford University Press, Oxford, UK, 1996.

\bibitem[Bartz-Beielstein(2005)]{bartz2005new}
T.~Bartz-Beielstein.
\newblock \emph{New experimentalism applied to evolutionary computation}.
\newblock PhD thesis, University of Dortmund, 2005.

\bibitem[Beielstein and Markon(2002)]{beielstein2002threshold}
T.~Beielstein and S.~Markon.
\newblock Threshold selection, hypothesis tests, and {DOE} methods.
\newblock In \emph{Proceedings of the IEEE Congress on Evolutionary Computation
  (CEC'02)}, pages 777--782, Honolulu, HI, 2002.

\bibitem[Beyer(2000)]{beyer2000evolutionary}
H.-G. Beyer.
\newblock Evolutionary algorithms in noisy environments: theoretical issues and
  guidelines for practice.
\newblock \emph{Computer Methods in Applied Mechanics and Engineering},
  186\penalty0 (2):\penalty0 239--267, 2000.

\bibitem[Chang et~al.(2006)Chang, Hou, and Su]{chang2006automated}
S.-J. Chang, H.-S. Hou, and Y.-K. Su.
\newblock Automated passive filter synthesis using a novel tree representation
  and genetic programming.
\newblock \emph{IEEE Transactions on Evolutionary Computation}, 10\penalty0
  (1):\penalty0 93--100, 2006.

\bibitem[Chang and Chen(2006)]{chang2006new}
Y.~Chang and S.~Chen.
\newblock A new query reweighting method for document retrieval based on
  genetic algorithms.
\newblock \emph{IEEE Transactions on Evolutionary Computation}, 10\penalty0
  (5):\penalty0 617--622, 2006.

\bibitem[Droste(2004)]{droste2004analysis}
S.~Droste.
\newblock Analysis of the (1+1) {EA} for a noisy {O}ne{M}ax.
\newblock In \emph{Proceedings of the 6th ACM Annual Conference on Genetic and
  Evolutionary Computation (GECCO'04)}, pages 1088--1099, Seattle, WA, 2004.

\bibitem[Droste et~al.(1998)Droste, Jansen, and
  Wegener]{droste:jansen:wegener:98}
S.~Droste, T.~Jansen, and I.~Wegener.
\newblock A rigorous complexity analysis of the (1+1) evolutionary algorithm
  for linear functions with {B}oolean inputs.
\newblock \emph{Evolutionary Computation}, 6\penalty0 (2):\penalty0 185--196,
  1998.

\bibitem[Droste et~al.(2002)Droste, Jansen, and Wegener]{droste2002analysis}
S.~Droste, T.~Jansen, and I.~Wegener.
\newblock {On the analysis of the (1+1) evolutionary algorithm}.
\newblock \emph{Theoretical Computer Science}, 276\penalty0 (1-2):\penalty0
  51--81, 2002.

\bibitem[Fitzpatrick and Grefenstette(1988)]{fitzpatrick1988genetic}
J.~M. Fitzpatrick and J.~J. Grefenstette.
\newblock Genetic algorithms in noisy environments.
\newblock \emph{Machine learning}, 3\penalty0 (2-3):\penalty0 101--120, 1988.

\bibitem[Fre{\v{\i}}dlin(1996)]{Freidlin:97}
M.~I. Fre{\v{\i}}dlin.
\newblock \emph{Markov Processes and Differential Equations: Asymptotic
  Problems}.
\newblock Birkh\"{a}user Verlag, Basel, Switzerland, 1996.

\bibitem[Freitas(2003)]{freitas2003survey}
A.~A. Freitas.
\newblock A survey of evolutionary algorithms for data mining and knowledge
  discovery.
\newblock In A.~Ghosh and S.~Tsutsui, editors, \emph{Advances in Evolutionary
  Computing: Theory and Applications}, pages 819--845. Springer-Verlag, New
  York, NY, 2003.

\bibitem[Goh and Tan(2007)]{goh2007investigation}
C.~Goh and K.~Tan.
\newblock An investigation on noisy environments in evolutionary multiobjective
  optimization.
\newblock \emph{IEEE Transactions on Evolutionary Computation}, 11\penalty0
  (3):\penalty0 354--381, 2007.

\bibitem[He and Yao(2001)]{YaoAI01}
J.~He and X.~Yao.
\newblock {Drift analysis and average time complexity of evolutionary
  algorithms}.
\newblock \emph{Artificial Intelligence}, 127\penalty0 (1):\penalty0 57--85,
  2001.

\bibitem[He and Yao(2004)]{he2004study}
J.~He and X.~Yao.
\newblock A study of drift analysis for estimating computation time of
  evolutionary algorithms.
\newblock \emph{Natural Computing}, 3\penalty0 (1):\penalty0 21--35, 2004.

\bibitem[Heidrich-Meisner and Igel(2009)]{heidrich2009hoeffding}
V.~Heidrich-Meisner and C.~Igel.
\newblock Hoeffding and {B}ernstein races for selecting policies in
  evolutionary direct policy search.
\newblock In \emph{Proceedings of the 26th International Conference on Machine
  Learning (ICML'09)}, pages 401--408, Montreal, Canada, 2009.

\bibitem[Jansen et~al.(2005)Jansen, Jong, and Wegener]{jansen2005choice}
T.~Jansen, K.~Jong, and I.~Wegener.
\newblock On the choice of the offspring population size in evolutionary
  algorithms.
\newblock \emph{Evolutionary Computation}, 13\penalty0 (4):\penalty0 413--440,
  2005.

\bibitem[Jin and Branke(2005)]{jin2005evolutionary}
Y.~Jin and J.~Branke.
\newblock Evolutionary optimization in uncertain environments-a survey.
\newblock \emph{IEEE Transactions on Evolutionary Computation}, 9\penalty0
  (3):\penalty0 303--317, 2005.

\bibitem[Ma et~al.(2006)Ma, Chan, Yao, and Chiu]{ma2006evolutionary}
P.~Ma, K.~Chan, X.~Yao, and D.~Chiu.
\newblock An evolutionary clustering algorithm for gene expression microarray
  data analysis.
\newblock \emph{IEEE Transactions on Evolutionary Computation}, 10\penalty0
  (3):\penalty0 296--314, 2006.

\bibitem[Markon et~al.(2001)Markon, Arnold, Back, Beielstein, and
  Beyer]{markon2001thresholding}
S.~Markon, D.~V. Arnold, T.~Back, T.~Beielstein, and H.-G. Beyer.
\newblock Thresholding-a selection operator for noisy {ES}.
\newblock In \emph{Proceedings of the IEEE Congress on Evolutionary Computation
  (CEC'01)}, pages 465--472, Seoul, Korea, 2001.

\bibitem[Neumann and Wegener(2007)]{neumann2007randomized}
F.~Neumann and I.~Wegener.
\newblock Randomized local search, evolutionary algorithms, and the minimum
  spanning tree problem.
\newblock \emph{Theoretical Computer Science}, 378\penalty0 (1):\penalty0
  32--40, 2007.

\bibitem[Qian et~al.(2012)Qian, Yu, and Zhou]{qian2012algorithm}
C.~Qian, Y.~Yu, and Z.-H. Zhou.
\newblock On algorithm-dependent boundary case identification for problem
  classes.
\newblock In \emph{Proceedings of the 12th International Conference on Parallel
  Problem Solving from Nature (PPSN'12)}, pages 62--71, Taormina, Italy, 2012.

\bibitem[Rudolph(2001)]{rudolph2001partial}
G.~Rudolph.
\newblock A partial order approach to noisy fitness functions.
\newblock In \emph{Proceedings of the IEEE Congress on Evolutionary Computation
  (CEC'01)}, pages 318--325, Seoul, Korea, 2001.

\bibitem[Sudholt(2013)]{sudholt2011new}
D.~Sudholt.
\newblock A new method for lower bounds on the running time of evolutionary
  algorithms.
\newblock \emph{IEEE Transactions on Evolutionary Computation}, 17\penalty0
  (3):\penalty0 418--435, 2013.

\bibitem[Yu and Zhou(2008)]{Yu:Zhou:08}
Y.~Yu and Z.-H. Zhou.
\newblock A new approach to estimating the expected first hitting time of
  evolutionary algorithms.
\newblock \emph{Artificial Intelligence}, 172\penalty0 (15):\penalty0
  1809--1832, 2008.

\end{thebibliography}
\bibliographystyle{abbrvnat}

\section*{Appendix}

\begin{myproofd}{Lemma \ref{lemma_analysis_condition}}
We prove it by induction on $m$.

{\bf (a) Initialization} is to prove that it holds when $m=1$.
\begin{equation}
\begin{aligned}
&\sum\nolimits^{1}_{i=0}P_iE_i =\sum\nolimits^{1}_{i=0}Q_iE_i+(P_0-Q_0)E_0+(P_1-Q_1)E_1\\
&\mathop{=}^{1}\sum\nolimits^{1}_{i=0}Q_iE_i+(P_0-Q_0)E_0+(1-P_0-(1-Q_0))E_1\\
&=\sum\nolimits^{1}_{i=0}Q_iE_i+(P_0-Q_0)(E_0-E_1)\geq \sum\nolimits^{1}_{i=0}Q_iE_i,
\end{aligned}
\end{equation}
where the `$\mathop{=}\limits^{1}$' is by $P_0+P_1=Q_0+Q_1=1$, and the `$\geq$' is by $P_0 \leq Q_0$ and $E_0<E_1$.

{\bf (b) Inductive Hypothesis} assumes that this lemma holds when $1\leq m \leq k$. Then, we consider $m=k+1$. The proof idea is to combine the first two terms of $\sum^{k+1}_{i=0}P_iE_i$, and then apply inductive hypothesis.

(1) When $P_0=P_1=0$, we can get
\begin{equation}
\begin{aligned}
&\sum\nolimits^{k+1}_{i=0}P_iE_i = (P_0+P_1)E_1 +\sum\nolimits^{k+1}_{i=2}P_iE_i\\
&\mathop{=}^{1}\sum\nolimits^{k}_{i=0} P'_i E'_i \geq^1 \sum\nolimits^{k}_{i=0} Q'_i E'_i\\ &\mathop{=}^{2}(Q_0+Q_1)E_1+ \sum\nolimits^{k+1}_{i=2}Q_iE_i \geq^2 \sum\nolimits^{k+1}_{i=0}Q_iE_i,
\end{aligned}
\end{equation}
where the `$\mathop{=}\limits^{1}$' and `$\mathop{=}\limits^{2}$' is by letting $E'_i=E_{i+1}$, $P'_0=P_0+P_1$, $Q'_0=Q_0+Q_1$ and $\forall i \geq 1, P'_i=P_{i+1}, Q'_i=Q_{i+1}$; the `$\mathop{\geq}^{1}$' is by applying inductive hypothesis because for $P'_i, Q'_i, E'_i$, the three conditions of this lemma hold and $m=k$; and the `$\mathop{\geq}^{2}$' is by $E_1>E_0$ and $Q_0 \geq 0$.

(2) When $P_0+P_1>0$, we consider two cases.\\
(2.1) If $P_1> Q_1$, we have
\begin{equation}
\begin{aligned}
&\sum\nolimits^{k+1}_{i=0}P_iE_i=(P_0+P_1)\frac{P_0E_0+P_1E_1}{P_0+P_1}+\sum\nolimits^{k+1}_{i=2}P_iE_i\\
&\geq^1 (Q_0+Q_1)\frac{P_0E_0+P_1E_1}{P_0+P_1}+\sum\nolimits^{k+1}_{i=2}Q_iE_i\\
&\geq^2 (Q_0+Q_1)\frac{Q_0E_0+Q_1E_1}{Q_0+Q_1}+\sum\nolimits^{k+1}_{i=2}Q_iE_i=\sum\nolimits^{k+1}_{i=0}Q_iE_i,
\end{aligned}
\end{equation}
where the `$\mathop{\geq}\nolimits^{1}$' is by applying inductive hypothesis as the `$\mathop{\geq}^{1}$' in case (1) except $E'_0=\frac{P_0E_0+P_1E_1}{P_0+P_1}$ here, and the `$\mathop{\geq}\nolimits^{2}$' can be easily derived by $Q_0\geq P_0, P_1>Q_1, E_1>E_0$.\\
(2.2) If $P_1\leq Q_1$, we have
\begin{equation}
\begin{aligned}
&\sum\nolimits^{k+1}_{i=0}P_iE_i=(P_0+P_1)\frac{P_0E_0+P_1E_1}{P_0+P_1}+\sum\nolimits^{k+1}_{i=2}P_iE_i\\
&\geq^1 (P_0+P_1)\frac{P_0E_0+P_1E_1}{P_0+P_1}+(Q_0-P_0+Q_1-P_1+Q_2) E_2+\sum\nolimits^{k+1}_{i=3}Q_iE_i\\
&\geq^2 (P_0+P_1)\frac{P_0E_0+P_1E_1}{P_0+P_1}+(Q_0-P_0)E_0+(Q_1-P_1)E_1+\sum\nolimits^{k+1}_{i=2}Q_iE_i\\
&=\sum\nolimits^{k+1}_{i=0}Q_iE_i,
\end{aligned}
\end{equation}
where the `$\mathop{\geq}^1$' is by applying inductive hypothesis as the `$\mathop{\geq}^{1}$' in case (1) except $E'_0=\frac{P_0E_0+P_1E_1}{P_0+P_1}$, $Q'_0=P_0+P_1$, $Q'_1=Q_0-P_0+Q_1-P_1+Q_2$ here, and the `$\mathop{\geq}^2$' is by $Q_0 \geq P_0$, $Q_1 \geq P_1$ and $E_2 >E_1>E_0$.

{\bf{(c) Conclusion}} According to (a) and (b), the lemma holds.
\end{myproofd}

\begin{myproofd}{Lemma \ref{CFHT_Trap}}
First, $\mathbb{E}_1(0)< \mathbb{E}_1(1)$ trivially holds, because $\mathbb{E}_1(0)=0$ and $\mathbb{E}_1(1)>0$. Then, we prove $\forall\; 0 < j <n:\mathbb{E}_1(j)<\mathbb{E}_1(j+1)$ inductively on $j$.

{\bf (a) Initialization} is to prove $\mathbb{E}_1(n-1) < \mathbb{E}_1(n)$. For $\mathbb{E}_1(n)$, because the next solution can be only $1^n$ or $0^n$, we have $\mathbb{E}_1(n)=1+(1-(1-p^n)^{\lambda})\mathbb{E}_1(0)+(1-p^n)^{\lambda}\mathbb{E}_1(n)$, then,
$\mathbb{E}_1(n)=1/(1-(1-p^n)^{\lambda})$. For $\mathbb{E}_1(n-1)$, because the next solution can be $1^n$, $0^n$ or a solution with $n-1$ number of 0 bits, we have $
\mathbb{E}_1(n-1)=1+(1-(1-p^{n-1}(1-p))^{\lambda})\mathbb{E}_1(0)+P\cdot\mathbb{E}_1(n)+((1-p^{n-1}(1-p))^{\lambda}-P)\mathbb{E}_1(n-1)$, where $P$ denotes the probability that the next solution is $0^n$. Then,
$\mathbb{E}_1(n-1)=(1+P\mathbb{E}_1(n))/(1-(1-p^{n-1}(1-p))^{\lambda}+P)$.
Thus, we have
$$\frac{\mathbb{E}_1(n-1)}{\mathbb{E}_1(n)}=\frac{1-(1-p^n)^{\lambda}+P}{1-(1-p^{n-1}(1-p))^{\lambda}+P}< 1,$$
where the inequality is by $0<p<0.5$.

{\bf (b) Inductive Hypothesis} assumes that
$$
\forall\; K< j \leq n-1 (K\geq 1): \mathbb{E}_1(j)< \mathbb{E}_1(j+1).
$$

Then, we consider $j=K$. Let $x$ and $x'$ be a solution with $K+1$ number of 0 bits and that with $K$ number of 0 bits, respectively. Then, we have $\mathbb{E}_1(K+1)=\expect{\tau'\mid \xi'_0=x}$ and $\mathbb{E}_1(K)=\expect{\tau'\mid \xi'_0=x'}$.

For the solution $x$, we divide the mutation on $x$ into two parts: mutation on one 0 bit and mutation on the $n-1$ remaining bits. The $n-1$ remaining bits contain $K$ number of 0 bits since $|x|_0=K+1$. Let $P^j_0$ and $P^j_i$ $(1 \leq i \leq n)$ be the probability that for the $\lambda$ offspring solutions under the condition that the 0 bit in the first mutation part is flipped by $j$ ($0\leq j \leq \lambda$) times in the $\lambda$ mutations, the least number of 0 bits is $0$, and the largest number of 0 bits is $i$ while the least number of 0 bits is larger than $0$, respectively. By considering the mutation and selection behavior of the (1+$\lambda$)-EA on the I$_{hardest}$ problem, we have, assuming that $\lambda$ is even,
\begin{equation}
\begin{aligned}
\mathbb{E}_1(K+1)
&=1\\
j: 0 \rightarrow \frac{\lambda}{2}-1&\begin{cases}
+& \cdots\\
+& \binom{\lambda}{j}p^{j}(1-p)^{\lambda-j}\cdot (P^{j}_0\mathbb{E}_1(0)+\sum\nolimits^{K}_{i=1}P^{j}_{i}\mathbb{E}_1(K+1)+\sum\nolimits^{n}_{i=K+1}P^{j}_{i}\mathbb{E}_1(i))\\
+& \cdots\\
\end{cases}\\
&+\binom{\lambda}{\lambda/2}p^{\frac{\lambda}{2}}(1-p)^{\frac{\lambda}{2}}\cdot (P^{\frac{\lambda}{2}}_0\mathbb{E}_1(0)+\sum\nolimits^{K}_{i=1}P^{\frac{\lambda}{2}}_{i}\mathbb{E}_1(K+1)+\sum\nolimits^{n}_{i=K+1}P^{\frac{\lambda}{2}}_{i}\mathbb{E}_1(i))\\
j: \frac{\lambda}{2}-1 \rightarrow 0&\begin{cases}
+& \cdots\\
+& \binom{\lambda}{\lambda-j}p^{\lambda-j}(1-p)^{j}\cdot (P^{\lambda-j}_0\mathbb{E}_1(0)+\sum\nolimits^{K}_{i=1}P^{\lambda-j}_{i}\mathbb{E}_1(K+1)+\sum\nolimits^{n}_{i=K+1}P^{\lambda-j}_{i}\mathbb{E}_1(i))\\
+& \cdots,
\end{cases}
\end{aligned}
\end{equation}
where the term $\binom{\lambda}{j}p^{j}(1-p)^{\lambda-j}$ ($0\leq j \leq \lambda$) is the probability that the 0 bit in the first mutation part is flipped by $j$ times in the $\lambda$ mutations.

For the solution $x'$, we also divide the mutation on $x'$ into two parts: mutation on one 1 bit and mutation on the $n-1$ remaining bits. The $n-1$ remaining bits also contain $K$ number of 0 bits since $|x'|_0=K$. Note that, the $P^j_0$ and $P^j_i$ $(1 \leq i \leq n)$ defined above are actually also the probability that for the $\lambda$ offspring solutions under the condition that the 1 bit in the first mutation part is flipped by $\lambda-j$ ($0\leq j \leq \lambda$) times in the $\lambda$ mutations, the least number of 0 bits is $0$, and the largest number of 0 bits is $i$ while the least number of 0 bits is larger than $0$, respectively. Then, we have
\begin{aligna}
\mathbb{E}_1(K)
&=1\\
j: 0 \rightarrow \frac{\lambda}{2}-1&\begin{cases}
+& \cdots\\
+& \binom{\lambda}{j}p^{j}(1-p)^{\lambda-j}\cdot (P^{\lambda-j}_0\mathbb{E}_1(0)+\sum\nolimits^{K}_{i=1}P^{\lambda-j}_{i}\mathbb{E}_1(K)+\sum\nolimits^{n}_{i=K+1}P^{\lambda-j}_{i}\mathbb{E}_1(i))\\
+& \cdots\\
\end{cases}\\
&+\binom{\lambda}{\lambda/2}p^{\frac{\lambda}{2}}(1-p)^{\frac{\lambda}{2}}\cdot (P^{\frac{\lambda}{2}}_0\mathbb{E}_1(0)+\sum\nolimits^{K}_{i=1}P^{\frac{\lambda}{2}}_{i}\mathbb{E}_1(K)+\sum\nolimits^{n}_{i=K+1}P^{\frac{\lambda}{2}}_{i}\mathbb{E}_1(i))\\
j: \frac{\lambda}{2}-1 \rightarrow 0&\begin{cases}
+& \cdots\\
+& \binom{\lambda}{\lambda-j}p^{\lambda-j}(1-p)^{j}\cdot (P^{j}_0\mathbb{E}_1(0)+\sum\nolimits^{K}_{i=1}P^{j}_{i}\mathbb{E}_1(K)+\sum\nolimits^{n}_{i=K+1}P^{j}_{i}\mathbb{E}_1(i))\\
+& \cdots,
\end{cases}
\end{aligna}
where the term $\binom{\lambda}{j}p^{j}(1-p)^{\lambda-j}$ ($0\leq j \leq \lambda$) is the probability that the 1 bit in the first mutation part is flipped by $j$ times in the $\lambda$ mutations.

From the above two equalities, we have
\begin{aligna}
&\mathbb{E}_1(K+1)-\mathbb{E}_1(K)=\\
j: 0\rightarrow \frac{\lambda}{2}-1 &\begin{cases}
&\cdots\\
&\begin{cases}
+&\binom{\lambda}{j}p^{j}(1-p)^{\lambda-j} \cdot \big(P^j_0 \mathbb{E}_1(0)-P^{\lambda-j}_0 \mathbb{E}_1(0)+\sum\limits^{n}_{i=K+1}P^{j}_i\mathbb{E}_1(i)-\sum\limits^{n}_{i=K+1}P^{\lambda-j}_i\mathbb{E}_1(i)\\
&+\sum\limits^{K}_{i=1}P^{j}_{i}\mathbb{E}_1(K+1)-\sum\limits^{K}_{i=1}P^{\lambda-j}_{i}\mathbb{E}_1(K+1)+\sum\limits^{K}_{i=1}P^{\lambda-j}_{i}\mathbb{E}_1(K+1)-\sum\limits^{K}_{i=1}P^{\lambda-j}_{i}\mathbb{E}_1(K) \big)\\
\end{cases}\\
&+\cdots\\
\end{cases}\\
&\quad+\binom{\lambda}{\lambda/2}p^{\frac{\lambda}{2}}(1-p)^{\frac{\lambda}{2}}\cdot\big(\sum\nolimits^{K}_{i=1}P^{\frac{\lambda}{2}}_i(\mathbb{E}_1(K+1)-\mathbb{E}_1(K))\big)\\
j: \frac{\lambda}{2}-1 \rightarrow 0&\begin{cases}
&+\cdots\\
&\begin{cases}
+&\binom{\lambda}{\lambda-j}p^{\lambda-j}(1-p)^{j}\cdot \big(P^{\lambda-j}_0 \mathbb{E}_1(0)-P^{j}_0 \mathbb{E}_1(0)+\sum\limits^{n}_{i=K+1}P^{\lambda-j}_i\mathbb{E}_1(i)-\sum\limits^{n}_{i=K+1}P^{j}_i\mathbb{E}_1(i)\\
&+\sum\limits^{K}_{i=1}P^{\lambda-j}_{i}\mathbb{E}_1(K+1)-\sum\limits^{K}_{i=1}P^{j}_{i}\mathbb{E}_1(K+1)+\sum\limits^{K}_{i=1}P^{j}_{i}\mathbb{E}_1(K+1)-\sum\limits^{K}_{i=1}P^{j}_{i}\mathbb{E}_1(K) \big)\\
\end{cases}\\
&+\cdots\\
\end{cases}\\
&=(\text{by combining the $j$-th and the $(\lambda-j)$-th term})\\
j: 0\rightarrow \frac{\lambda}{2}-1 &\begin{cases}
&\cdots\\
&\begin{cases}
+&\big(\binom{\lambda}{j}p^{j}(1-p)^{\lambda-j}-\binom{\lambda}{\lambda-j}p^{\lambda-j}(1-p)^{j}\big) \cdot \big(P^j_0 \mathbb{E}_1(0)+\sum\limits^{K+1}_{i=1}P^{j}_{i}\mathbb{E}_1(K+1)\\
&+\sum\limits^{n}_{i=K+2}P^{j}_i\mathbb{E}_1(i)-P^{\lambda-j}_0 \mathbb{E}_1(0)-\sum\limits^{K+1}_{i=1}P^{\lambda-j}_{i}\mathbb{E}_1(K+1)-\sum\limits^{n}_{i=K+2}P^{\lambda-j}_i\mathbb{E}_1(i)\big)\\
+&\binom{\lambda}{j}p^{j}(1-p)^{\lambda-j}\cdot\big(\sum\limits^{K}_{i=1}P^{\lambda-j}_{i}(\mathbb{E}_1(K+1)-\mathbb{E}_1(K))\big)\\
+&\binom{\lambda}{\lambda-j}p^{\lambda-j}(1-p)^{j}\cdot\big(\sum\limits^{K}_{i=1}P^{j}_{i}(\mathbb{E}_1(K+1)-\mathbb{E}_1(K))\big)\\
\end{cases}\\
&+\cdots\\
\end{cases}\\
&\quad+\binom{\lambda}{\lambda/2}p^{\frac{\lambda}{2}}(1-p)^{\frac{\lambda}{2}}\cdot\big(\sum\nolimits^{K}_{i=1}P^{\frac{\lambda}{2}}_i(\mathbb{E}_1(K+1)-\mathbb{E}_1(K))\big).
\end{aligna}

Then, we are to investigate the relation between $\sum^k_{i=0} P^j_i$ and $\sum^k_{i=0} P^{\lambda-j}_i$ for $0\leq j\leq \frac{\lambda}{2}-1$. Let $m$ ($0 \leq m\leq n-1$) denote the number of 0 bits after bit-wise mutation on a Boolean string of length $n-1$ with $K$ number of 0 bits. For the $\lambda$ independent mutations, we use $m_1,\ldots,m_{\lambda}$, respectively. By the definition of $P^j_i$, we know that there are $j$ number of 1 bits in the first mutation part, since $j$ 0 bits are flipped in the $\lambda$ mutations. Under this condition, $\sum^k_{i=0} P^j_i$ is the probability that for the $\lambda$ offspring solutions, the least number of 0 bits is 0, or the least number of 0 bits is larger than 0 while the largest number of 0 bits is not larger than $k$. We assume that the $j$ number of 1 bits in the first mutation part correspond to $m_1,\ldots,m_j$. Thus, we have
\begin{aligna}
\sum^k_{i=0} P^j_i=&P\big( m_1=0 \vee \ldots \vee m_j=0 \\
&\vee (0 <m_1 \leq k \wedge \ldots \wedge 0<m_j \leq k \wedge m_{j+1} \leq k-1 \wedge \ldots \wedge m_{\lambda} \leq k-1)\big),
\end{aligna}
and
\begin{aligna}
\sum^k_{i=0} P^{\lambda-j}_i=&P\big(m_1=0 \vee \ldots \vee m_{\lambda-j}=0 \\
& \vee (0<m_1 \leq k \wedge \ldots \wedge 0<m_{\lambda-j} \leq k \wedge m_{\lambda-j+1} \leq k-1 \wedge \ldots \wedge m_{\lambda} \leq k-1)\big)\\
\geq &P\big( m_1=0 \vee \ldots \vee m_j=0 \vee (0 <m_1 \leq k \wedge \ldots \wedge 0<m_j \leq k\\
& \wedge m_{j+1} \leq k \wedge \ldots \wedge m_{\lambda-j} \leq k \wedge m_{\lambda-j+1} \leq k-1 \wedge \ldots \wedge m_{\lambda} \leq k-1)\big).
\end{aligna}
Then, we have
\begin{aligna}\label{trap-prob1}
&\forall 0 \leq k \leq n-1, \quad \sum^k_{i=0} P^j_i\leq \sum^k_{i=0} P^{\lambda-j}_i.
\end{aligna}

By Lemma \ref{lemma_analysis_condition}, we can get
$$
P^j_0 \mathbb{E}_1(0)+\sum\limits^{K+1}_{i=1}P^{j}_{i}\mathbb{E}_1(K+1)+\sum\limits^{n}_{i=K+2}P^{j}_i\mathbb{E}_1(i)\geq P^{\lambda-j}_0\mathbb{E}_1(0)+\sum\limits^{K+1}_{i=1}P^{\lambda-j}_{i}\mathbb{E}_1(K+1)+\sum\limits^{n}_{i=K+2}P^{\lambda-j}_i\mathbb{E}_1(i).
$$
The three conditions of Lemma \ref{lemma_analysis_condition} can be easily verified, because $\mathbb{E}_1(0)=0<\mathbb{E}_1(K+1)<\ldots<\mathbb{E}_1(n)$ by inductive hypothesis; $\sum^{n}_{i=0} P^j_i= \sum^{n}_{i=0} P^{\lambda-j}_i=1$; and Eq.\refeq{trap-prob1} holds.

By the above inequality and $p<0.5$, we have
\begin{aligna}
&\mathbb{E}_1(K+1)-\mathbb{E}_1(K)> (\sum^{\lambda}_{j=0}\binom{\lambda}{j}p^{j}(1-p)^{\lambda-j}\sum\limits^{K}_{i=1}P^{\lambda-j}_i)\cdot \big(\mathbb{E}_1(K+1)-\mathbb{E}_1(K)\big).
\end{aligna}

Because $\sum^{\lambda}_{j=0}\binom{\lambda}{j}p^{j}(1-p)^{\lambda-j}\sum\limits^{K}_{i=1}P^{\lambda-j}_i<\sum^{\lambda}_{j=0}\binom{\lambda}{j}p^{j}(1-p)^{\lambda-j}=1$, we have $\mathbb{E}_1(K+1)>\mathbb{E}_1(K)$.

For the case that $\lambda$ is odd, we can prove it similarly.

{\bf{(c) Conclusion}} According to (a) and (b), the lemma holds.
\end{myproofd}

\begin{myproofd}{Lemma \ref{CFHT_OneMax}}
We prove $\forall \;0 \leq j <n:\mathbb{E}_2(j)<\mathbb{E}_2(j+1)$ inductively on $j$.

{\bf (a) Initialization} is to prove $\mathbb{E}_2(0) < \mathbb{E}_2(1)$, which trivially holds since $\mathbb{E}_2(1)>0=\mathbb{E}_2(0)$.

{\bf (b) Inductive Hypothesis} assumes that
$$
\forall \; 0 \leq j < K (K\leq n-1): \mathbb{E}_2(j)<\mathbb{E}_2(j+1).
$$

Then, we consider $j=K$. When comparing $\mathbb{E}_2(K+1)$ with $\mathbb{E}_2(K)$, we use the similar analysis method as that in the proof of Lemma \ref{CFHT_Trap}. Let $P^j_i \; (0 \leq i \leq n)$ be the probability that the least number of 0 bits for the $\lambda$ offspring solutions is $i$ under the condition that the 0 bit in the first mutation part is flipped by $j$ ($0 \leq j \leq \lambda$) times in the $\lambda$ mutations. Then, by considering the mutation and selection behavior of the (1+$\lambda$)-EA on the I$_{easiest}$ problem, we have, assuming that $\lambda$ is even,
\begin{aligna}
\mathbb{E}_2(K+1)
&=1\\
j: 0 \rightarrow \frac{\lambda}{2}-1&\begin{cases}
+& \cdots\\
+& \binom{\lambda}{j}p^{j}(1-p)^{\lambda-j}\cdot (\sum\nolimits^{K}_{i=0}P^{j}_i\mathbb{E}_2(i)+\sum\nolimits^{n}_{i=K+1}P^{j}_{i}\mathbb{E}_2(K+1))\\
+& \cdots\\
\end{cases}\\
&+\binom{\lambda}{\lambda/2}p^{\frac{\lambda}{2}}(1-p)^{\frac{\lambda}{2}}\cdot (\sum\nolimits^{K}_{i=0}P^{\frac{\lambda}{2}}_i\mathbb{E}_2(i)+\sum\nolimits^{n}_{i=K+1}P^{\frac{\lambda}{2}}_{i}\mathbb{E}_2(K+1))\\
j: \frac{\lambda}{2}-1 \rightarrow 0&\begin{cases}
+& \cdots\\
+& \binom{\lambda}{\lambda-j}p^{\lambda-j}(1-p)^{j}\cdot (\sum\nolimits^{K}_{i=0}P^{\lambda-j}_i\mathbb{E}_2(i)+\sum\nolimits^{n}_{i=K+1}P^{\lambda-j}_{i}\mathbb{E}_2(K+1))\\
+& \cdots,
\end{cases}
\end{aligna}
and
\begin{aligna}
\mathbb{E}_2(K)&=1\\
j: 0 \rightarrow \frac{\lambda}{2}-1&\begin{cases}
+& \cdots\\
+& \binom{\lambda}{j}p^{j}(1-p)^{\lambda-j}\cdot (\sum\nolimits^{K}_{i=0}P^{\lambda-j}_i\mathbb{E}_2(i)+\sum\nolimits^{n}_{i=K+1}P^{\lambda-j}_{i}\mathbb{E}_2(K))\\
+& \cdots\\
\end{cases}\\
&+\binom{\lambda}{\lambda/2}p^{\frac{\lambda}{2}}(1-p)^{\frac{\lambda}{2}}\cdot (\sum\nolimits^{K}_{i=0}P^{\frac{\lambda}{2}}_i\mathbb{E}_2(i)+\sum\nolimits^{n}_{i=K+1}P^{\frac{\lambda}{2}}_{i}\mathbb{E}_2(K))\\
j: \frac{\lambda}{2}-1 \rightarrow 0&\begin{cases}
+& \cdots\\
+& \binom{\lambda}{\lambda-j}p^{\lambda-j}(1-p)^{j}\cdot (\sum\nolimits^{K}_{i=0}P^{j}_i\mathbb{E}_2(i)+\sum\nolimits^{n}_{i=K+1}P^{j}_{i}\mathbb{E}_2(K))\\
+& \cdots.
\end{cases}
\end{aligna}

From the above two equalities, we have
\begin{aligna}
&\mathbb{E}_2(K+1)-\mathbb{E}_2(K)=\\
j: 0 \rightarrow \frac{\lambda}{2}-1 &\begin{cases}
&\cdots\\
&\begin{cases}
+&\binom{\lambda}{j}p^{j}(1-p)^{\lambda-j} \cdot \big(\sum\limits^{K}_{i=0}P^{j}_i\mathbb{E}_2(i)-\sum\limits^{K}_{i=0}P^{\lambda-j}_i\mathbb{E}_2(i)+\sum\limits^{n}_{i=K+1}P^{j}_{i}\mathbb{E}_2(K+1)\\
&-\sum\limits^{n}_{i=K+1}P^{j}_{i}\mathbb{E}_2(K)+\sum\limits^{n}_{i=K+1}P^{j}_{i}\mathbb{E}_2(K)-\sum\limits^{n}_{i=K+1}P^{\lambda-j}_{i}\mathbb{E}_2(K) \big)\\
\end{cases}\\
&+\cdots\\
\end{cases}\\
&\quad+\binom{\lambda}{\lambda/2}p^{\frac{\lambda}{2}}(1-p)^{\frac{\lambda}{2}}\cdot\big(\sum\nolimits^{n}_{i=K+1}P^{\frac{\lambda}{2}}_i(\mathbb{E}_2(K+1)-\mathbb{E}_2(K))\big)\\
j: \frac{\lambda}{2}-1 \rightarrow 0&\begin{cases}
&+\cdots\\
&\begin{cases}
+&\binom{\lambda}{\lambda-j}p^{\lambda-j}(1-p)^{j}\cdot \big(\sum\limits^{K}_{i=0}P^{\lambda-j}_i\mathbb{E}_2(i)-\sum\limits^{K}_{i=0}P^{j}_i\mathbb{E}_2(i)+\sum\limits^{n}_{i=K+1}P^{\lambda-j}_{i}\mathbb{E}_2(K+1)\\
&-\sum\limits^{n}_{i=K+1}P^{\lambda-j}_{i}\mathbb{E}_2(K)+\sum\limits^{n}_{i=K+1}P^{\lambda-j}_{i}\mathbb{E}_2(K)-\sum\limits^{n}_{i=K+1}P^{j}_{i}\mathbb{E}_2(K) \big)\\
\end{cases}\\
&+\cdots\\
\end{cases}\\
&= (\text{by combining the $j$-th and the $(\lambda-j)$-th term})\\
j: 0\rightarrow \frac{\lambda}{2}-1 &\begin{cases}
&\cdots\\
&\begin{cases}
+&\big(\binom{\lambda}{j}p^{j}(1-p)^{\lambda-j}-\binom{\lambda}{\lambda-j}p^{\lambda-j}(1-p)^{j}\big) \cdot \big(\sum\limits^{K-1}_{i=0}P^{j}_i\mathbb{E}_2(i)+\sum\limits^{n}_{i=K}P^{j}_{i}\mathbb{E}_2(K)\\
&-\sum\limits^{K-1}_{i=0}P^{\lambda-j}_i\mathbb{E}_2(i)-\sum\limits^{n}_{i=K}P^{\lambda-j}_{i}\mathbb{E}_2(K)\big)\\
+&\binom{\lambda}{j}p^{j}(1-p)^{\lambda-j}\cdot\big(\sum\limits^{n}_{i=K+1}P^{j}_{i}(\mathbb{E}_2(K+1)-\mathbb{E}_2(K))\big)\\
+&\binom{\lambda}{\lambda-j}p^{\lambda-j}(1-p)^{j}\cdot\big(\sum\limits^{n}_{i=K+1}P^{\lambda-j}_{i}(\mathbb{E}_2(K+1)-\mathbb{E}_2(K))\big)\\
\end{cases}\\
&+\cdots\\
\end{cases}\\
&\quad+\binom{\lambda}{\lambda/2}p^{\frac{\lambda}{2}}(1-p)^{\frac{\lambda}{2}}\cdot\big(\sum\nolimits^{n}_{i=K+1}P^{\frac{\lambda}{2}}_i(\mathbb{E}_2(K+1)-\mathbb{E}_2(K))\big).
\end{aligna}

Then, we are to investigate the relation between $\sum^k_{i=0} P^j_i$ and $\sum^k_{i=0} P^{\lambda-j}_i$ for $0\leq j\leq \frac{\lambda}{2}-1$. Let $m$ ($0 \leq m\leq n-1$) denote the number of 0 bits after bit-wise mutation on a Boolean string of
length $n-1$ with $K$ number of 0 bits. For the $\lambda$ independent mutations, we use $m_2,\ldots,m_{\lambda}$, respectively. By the definition of $P^j_i$, we know that there are $j$ number of 1 bits in the first mutation part. Under this condition, $\sum^k_{i=0} P^j_i$ is the probability that the least number of 0 bits for the $\lambda$ offspring solutions is not larger than $k$. We assume that the $j$ number of 1 bits in the first mutation part correspond to $m_1,\ldots,m_j$. Thus, we have
$$\sum^k_{i=0} P^j_i=P(m_1 \leq k \vee \ldots \vee m_j \leq k \vee m_{j+1} \leq k-1 \vee \ldots \vee m_{\lambda} \leq k-1),$$
and
$$\sum^k_{i=0} P^{\lambda-j}_i=P(m_1 \leq k \vee \ldots \vee m_{\lambda-j} \leq k \vee m_{\lambda-j+1} \leq k-1 \vee \ldots \vee m_{\lambda} \leq k-1).$$
Since $0\leq j\leq \frac{\lambda}{2}-1$, we have
\begin{aligna}\label{onemax-prob1}
&\forall 0 \leq k \leq n-1, \quad \sum^k_{i=0} P^j_i\leq \sum^k_{i=0} P^{\lambda-j}_i.
\end{aligna}

By Lemma \ref{lemma_analysis_condition}, we can get
$$
\sum\limits^{K-1}_{i=0}P^{j}_i\mathbb{E}_2(i)+\sum\limits^{n}_{i=K}P^{j}_{i}\mathbb{E}_2(K) \geq \sum\limits^{K-1}_{i=0}P^{\lambda-j}_i\mathbb{E}_2(i)+\sum\limits^{n}_{i=K}P^{\lambda-j}_{i}\mathbb{E}_2(K).
$$
The three conditions of Lemma \ref{lemma_analysis_condition} can be easily verified, because $\mathbb{E}_2(0)<\mathbb{E}_2(1)<\ldots<\mathbb{E}_2(K)$ by inductive hypothesis; $\sum^{n}_{i=0} P^j_i= \sum^{n}_{i=0} P^{\lambda-j}_i=1$; and Eq.\refeq{onemax-prob1} holds.

By the above inequality and $p<0.5$, we have
\begin{aligna}
&\mathbb{E}_2(K+1)-\mathbb{E}_2(K)> (\sum^{\lambda}_{j=0}\binom{\lambda}{j}p^{j}(1-p)^{\lambda-j}\sum\limits^{n}_{i=K+1}P^j_i)\cdot \big(\mathbb{E}_2(K+1)-\mathbb{E}_2(K)\big).
\end{aligna}

Because $\sum^{\lambda}_{j=0}\binom{\lambda}{j}p^{j}(1-p)^{\lambda-j}\sum\limits^{n}_{i=K+1}P^j_i<\sum^{\lambda}_{j=0}\binom{\lambda}{j}p^{j}(1-p)^{\lambda-j}=1$, we have $\mathbb{E}_2(K+1)>\mathbb{E}_2(K)$.

For the case that $\lambda$ is odd, we can prove it similarly.

{\bf{(c) Conclusion}} According to (a) and (b), the lemma holds.
\end{myproofd}

\end{document}